\documentclass{article}
\usepackage[final]{neurips_2022}
\usepackage[utf8]{inputenc} % allow utf-8 input
\usepackage[T1]{fontenc}    % use 8-bit T1 fonts
\usepackage{hyperref}       % hyperlinks
\usepackage{url}            % simple URL typesetting
\usepackage{booktabs}       % professional-quality tables
\usepackage{amsfonts}       % blackboard math symbols
\usepackage{nicefrac}       % compact symbols for 1/2, etc.
\usepackage{microtype}      % microtypography
\usepackage{xcolor}         % colors
\usepackage[labelfont=bf]{caption}
\usepackage{subcaption}
\usepackage{algorithm}
\usepackage{algpseudocode}
\usepackage{amsthm}
\usepackage{amsmath}
\usepackage{amsfonts}
\usepackage{mathtools}
\usepackage{dsfont}
\usepackage[capitalise]{cleveref}
\usepackage{lipsum}
\usepackage{caption}
\usepackage{subcaption}
\usepackage{tabularx}
\usepackage{graphicx}

\graphicspath{{./images/}}

\newtheorem{theorem}{Theorem}
\newtheorem{lemma}[theorem]{Lemma}
\newtheorem{proposition}[theorem]{Proposition}

\newtheorem{dfn}{Definition}
\newtheorem{assumption}{Assumption}

\newcommand{\E}{\mathbb{E}}

\renewcommand{\P}{\mathbb{P}}
\newcommand{\e}{\mathrm{e}}
\renewcommand{\epsilon}{\varepsilon}
\newcolumntype{Y}{>{\centering\arraybackslash}X}
\newcommand{\argmax}{\mathrm{arg\,max}}
\newcommand{\argmin}{\mathrm{arg\,min}}

\title{Trading Off Resource Budgets For \\Improved Regret Bounds}

\author{%
    Damon Falck$^{*}$ \\
    University of Oxford \\
    \texttt{damon.falck@gmail.com} \\
     \And
     Thomas Orton$^{*}$ \\
     University of Oxford \\
     \texttt{thomas.orton@cs.ox.ac.uk}
}

\begin{document}
\def\thefootnote{*}\footnotetext{Equal contribution.}

\maketitle

\begin{abstract}
  In this work we consider a variant of adversarial online learning where in each round one picks $B$ out of $N$ arms and incurs cost equal to the \textit{minimum} of the costs of each arm chosen. We propose an algorithm called Follow the Perturbed Multiple Leaders (FPML) for this problem, which we show (by adapting the techniques of \cite{DBLP:journals/jcss/KalaiV05}) achieves expected regret $\mathcal{O}(T^{\frac{1}{B+1}}\ln(N)^{\frac{B}{B+1}})$ over time horizon $T$ relative to the \textit{single} best arm in hindsight. This introduces a trade-off between the budget $B$ and the single-best-arm regret, and we proceed to investigate several applications of this trade-off. First, we observe that algorithms which use standard regret minimizers as subroutines can sometimes be adapted by replacing these subroutines with FPML, and we use this to generalize existing algorithms for Online Submodular Function Maximization \citep{DBLP:conf/nips/StreeterG08} in both the full feedback and semi-bandit feedback settings. Next, we empirically evaluate our new algorithms on an online black-box hyperparameter optimization problem. Finally, we show how FPML can lead to new algorithms for Linear Programming which require stronger oracles at the benefit of fewer oracle calls. 
\end{abstract}

\section{Introduction}

Adversarial online learning is a well-studied framework for sequential decision making with numerous applications. In each round $t=1,\dots,T$, an adversary chooses a hidden cost function $c_{t}:\mathcal{A} \rightarrow [0,1]$ from a set of arms $\mathcal{A}$ to costs in $[0,1]$. An algorithm must then choose an arm $a_t \in \mathcal{A}$, and incurs cost $c_t(a_t)$. In the \textit{full feedback} setting (Online Learning with Experts (OLwE)), the algorithm then observes the cost function $c_t$. In the \textit{partial feedback} setting (Multi-Armed Bandits (MAB)) the algorithm only observes the incurred cost $c_t(a_t)$. The objective is to find algorithms which minimize \textit{regret}, defined as the difference between the algorithm's cumulative cost and the cumulative cost of the single best arm in hindsight.

In this paper we consider a search-like variant of these problems where in each round one can pick a \textit{subset} of arms $S_t \subset \mathcal{A}$ with $|S_{t}|=B \geq 1$, and \textit{keep the arm with the smallest cost}. This variant appears naturally in many settings, including:

\begin{enumerate}
    \item Online algorithm portfolios \citep{DBLP:journals/ai/GomesS01}: In each round $t$, one receives a problem instance $x_t$, and can pick a subset $S_t \subset \mathcal{A}$ of algorithms to run in parallel to solve $x_t$. For example, $x_t$ could be a boolean satisfiability (SAT) problem, and $\mathcal{A}$ could be a collection of different SAT solving algorithms.  We let $c_{t}(a)=0$ if $a$ solves $x_t$ and $c_{t}(a)=1$ otherwise. Then if any $a \in S_t$ finds a solution to $x_t$ we incur $0$ cost in this round. Another example is online hyperparameter optimization (see Section \ref{sec:experiments}).
    
    \item Online bidding \citep{DBLP:conf/nips/ChenHLLLL16}: In each round $t$, an auctioneer sets up a first-price auction for bidders $S_t \subset \mathcal{A}$. Each bidder $a \in \mathcal{A}$ has a price $1-c_{t}(a)$ they are willing to pay, and the auctioneer receives $\max_{a \in S_t} 1-c_{t}(a)=1-\min_{a \in S_t} c_{t}(a)$, and so maximizing revenue is equivalent to minimizing costs. 
    
    \item Adaptive network routing \citep{awerbuch2008online}: In each round $t$, a network router receives a data packet $x_t$ and can pick a selection of network routes $S_t \subset \mathcal{A}$ to send it to its destination via in parallel. Let the cost $c_t(a)$ of a route $a$ be the total time taken for $x_t$ to reach its destination via $a$; the router receives cost $\min_{a \in S_t} c_t(a)$ equal to the smallest delay.

\end{enumerate}

In many applications the budget $B$ is a restricted resource (e.g. compute time or number of cores) we would like to keep small; this paper studies how one can trade off budget resources for better guarantees on the standard regret objective.

Formally, for any randomized algorithm \textbf{ALG} which chooses subset $S_t \subset \mathcal{A}$ in round $t$, and thus incurs cost $c_t(S_t):=\min_{a \in S_t}c_t(a)$, define
\begin{align*}
    R_{T}^{*}(\text{\textbf{ALG}}) := \max_{ c_1,\dots,c_T}\mathop{\mathbb{E}} \left[\sum_{t=1}^T c_{t}(S_t)-\min_{a^{*} \in \mathcal{A}} \sum_{t=1}^{T} c_{t}(a^{*})\right]
\end{align*}
to be the worst-case expected regret of \textbf{ALG} relative to the single best arm in hindsight, where the expectation is with respect to the randomness of \textbf{ALG}.\footnote{Here we consider an oblivious adversary model for simplicity, but we believe the results of this paper carry through to adaptive adversaries as well.} What guarantees can we give on $R_{T}^{*}$ as a function of our budget $B$? In the full feedback setting when $B=1$, this is the standard OLwE problem where it is known that $R_{T}^{*}=\Omega(\sqrt{T})$ and there are algorithms which achieve $R_{T}^{*} \leq 2\sqrt{T \ln(N)}$ \citep{lattimore2020bandit}, where $N = |\mathcal{A}|$. When $B=N$ the algorithm which chooses $S_t=\mathcal{A}$ in each round achieves $R_{T}^{*}=0$. But what bounds on $R_{T}^{*}$ can be achieved in the intermediate regime when $1<B<N$? To the best of our knowledge this question has not been directly answered by any prior work. 

\subsection{Contributions} \label{sec:contrib}

\paragraph{Theoretical results:} We present a new algorithm for this learning problem called \textit{Follow the Perturbed Multiple Leaders} (\textbf{FPML}), and show that in the full feedback setting $R_{T}^{*}(\text{\textbf{FPML}})\leq \mathcal{O}(T^{\frac{1}{B+1}}\ln(N)^{\frac{B}{B+1}})$. This allows for a direct trade-off between the budget $B$ and the regret bound (in particular, allowing resources $B\ge\Omega(\ln(T))$ leads to regret \textit{constant} in $T$) and recovers the familiar $\mathcal{O}(\sqrt{T \ln(N)})$ bound when $B=1$. We then show that in the \textit{semi-bandit feedback} setting (where the algorithm finds out only the costs of the arms it chooses) this bound can be converted to $R_{T}^{*}(\text{\textbf{FPML}})\leq \mathcal{O}( T^{\frac{1}{B+1}}(K\ln(N))^{\frac{B}{B+1}})$ if one has unbiased cost estimators bounded in $[0,K]$.

We also consider the more general problem of Online Submodular Function Maximization (OSFM), for which prior work gives an online greedy algorithm \textbf{OG} \citep{DBLP:conf/nips/StreeterG08}. When given a budget of $B$ per round, \textbf{OG} achieves regret $\mathcal{O}(\sqrt{TB\ln(N)})$ with respect to $(1-\textit{e}^{-1})$OPT($B$), where OPT($B$) is the performance of the best fixed length-$B$ schedule (see Section \ref{sec:application} for a formal definition of OSFM). Note that in this guarantee the regret benchmark is a function of the algorithm budget. By replacing a subroutine in \textbf{OG} with \textbf{FPML}, we generalize \textbf{OG} to a new algorithm \textbf{OG\textsubscript{hybrid}}. Unlike \textbf{OG}, \textbf{OG\textsubscript{hybrid}} is able to give regret bounds against benchmarks which are decoupled from the algorithm budget. This allows one to more easily quantify the trade-off of increasing the budget against a fixed regret objective. As a special case, we are able to show that having a budget of $B = B'\lceil\ln(T)^2\rceil$ per round allows one to achieve regret $\mathcal{O}(B'\ln(T)\ln(N))$ with respect to OPT($B'$). One interpretation of this result is that if you are willing to increase your budget (e.g. runtime) by a factor of $\ln(T)^2$, you are able to improve your performance guarantee benchmark from $(1-\textit{e}^{-1})$OPT($B'$) to  OPT($B'$). Likewise, your regret growth rate in terms of the number of rounds changes from $\mathcal{O}(\sqrt{T})$ to $\mathcal{O}(\ln(T))$. 

Finally, in Section \ref{sec:lp}  we show how to use \textbf{FTML} to generalize a technique for solving linear programs assuming access to an oracle which solves relaxed forms of the linear program. To obtain an $\epsilon$-approximate solution to the linear program requires $\left(\frac{1}{\epsilon}\right)^{\frac{B+1}{B}} (4\rho)^{\frac{B+1}{B}} (1+\ln(n))$ oracle calls, where the parameters $(B,\rho)$ are related to the power of the oracle and $n$ is the number of linear constraints. The case $B=1$ coincides with known results. 

\paragraph{Experimental results:} We benchmark both \textbf{FPML} and \textbf{OG\textsubscript{hybrid}} on an online black-box hyperparameter optimization problem based on the 2020 NeurIPS BBO challenge \citep{turner2021bayesian}. We find that both these new algorithms outperform \textbf{OG} for various compute budgets. We are able to explain why this happens for this specific dataset, and discuss the scenarios under which each algorithm would perform better. 

\paragraph{Techniques:}

Minimizing $R_{T}^{*}$ is an important subroutine for a large variety of applications including Linear Programming, Boosting, and solving zero sum games \citep{DBLP:journals/toc/AroraHK12}. Traditionally an experts algorithm such as \textbf{Hedge} \citep{DBLP:journals/iandc/LittlestoneW94}, which pulls a single arm per round, will be used as a subroutine to minimize $R_{T}^{*}$. We highlight how in the cases of OSFM and Linear Programming, one can simply replace a single arm $R_{T}^{*}$-minimizing subroutine with \textbf{FPML} and get performance bounds with little or no alteration to the original proofs. The resulting algorithms have improved bounds (due to improved bounds on $R_{T}^{*}$ when $B>1$) at the cost of qualitatively changing the algorithm (e.g. requiring a larger budget or more powerful oracle). This is significant because it highlights the potential of how bounds on $R_{T}^{*}$ when $B>1$ can lead to new results in other application areas. In Section \ref{sec:FPML} we also highlight how the proof techniques of \cite{DBLP:journals/jcss/KalaiV05} for bounding $R_{T}^{*}$ in the traditional experts setting can naturally be generalized to the case when $B>1$, which is of independent interest.

\subsection{Relation to prior work} \label{sec:prior}

One can alternatively formulate more gewe consider as receiving the maximum \textit{reward} $r_{t}(a)=1-c_{t}(a)$ of each arm chosen instead of the minimum cost. In this maximum of rewards formulation, the problem fits within the OSFM framework where (a) all actions are unit-time and (b) the submodular \textit{job} function is always a maximum of rewards. The rewards formulation of the problem has also been separately studied as the K-MAX problem (here $K=B$) \cite{DBLP:conf/nips/ChenHLLLL16}. In the OSFM setting, \cite{DBLP:conf/nips/StreeterG08} give an online greedy approximation algorithm which guarantees $\mathop{\mathbb{E}}[(1-e^{-1})\text{OPT}(B)-\text{Reward}_{T}] \leq \mathcal{O}(\sqrt{TB\ln(N)})$ in the full feedback adversarial setting, where $\text{OPT}(B)$ is the cumulative reward of the best fixed subset of $B$ arms in hindsight, and $\text{Reward}_{T}$ is the cumulative reward of the algorithm. A similar bound of $\mathcal{O}(B\sqrt{TN\ln(N)})$ can be given in a semi-feedback setting. Conversely in the full feedback setting, \cite{streeter-detailed} shows that any algorithm has worst-case regret $\mathop{\mathbb{E}}[\text{OPT}(B)-\text{Reward}_{T}] \geq \Omega(\sqrt{TB\ln(N/B)})$ when one receives the maximum of rewards in each round. \cite{DBLP:conf/nips/ChenHLLLL16} study the K-MAX problem and other non-linear reward functions in the stochastic combinatorial multi-armed bandit setting. Assuming the rewards satisfy certain distributional assumptions, they give an algorithm which achieves distribution-independent regret bounds of $\mathop{\mathbb{E}}[(1-\epsilon) \text{OPT}(B)-\text{Reward}_{T}] \leq \mathcal{O}(\sqrt{TBN\ln(T)})$ for $\epsilon>0$ with semi-bandit feedback. Note however that we consider the adversarial setting in this paper. 

More broadly, these problems fall within the combinatorial online learning setting where an algorithm may pull a subset of arms in each round. Much prior work has focused on combinatorial bandits where the reward is linear in the subset of arms chosen, which can model applications including online advertising and online shortest paths \citep{DBLP:journals/jcss/Cesa-BianchiL12,DBLP:journals/mor/AudibertBL14,DBLP:conf/nips/CombesSPL15}. The case of non-linear reward is comparatively less studied, but having non-linear rewards (such as max) allows one to model a wider variety of problems including online expected utility maximization \citep{DBLP:conf/focs/LiD11,DBLP:conf/nips/ChenHLLLL16}. As some examples of prior work in the stochastic setting, \citep{DBLP:conf/icml/GopalanMM14} uses Thompson Sampling to deal with non-linear rewards of functions of subsets of arms (including the max function), but requires the rewards to come from a known parametric distribution. \cite{DBLP:journals/jmlr/ChenWYW16} considers a model where the subset of arms pulled is randomized based on pulling a `super-arm', and the reward is a non-linear function of the values of the arms pulled. In the adversarial setting, \cite{han2021adversarial} study the combinatorial MAB problem when rewards can be expressed as a $d$-degree polynomial. 

In contrast to prior work which focuses on giving algorithms which compete against benchmarks which have \textit{the same} budget as the algorithm, this work is concerned with the trade-off between regret bounds and budget size. We focus on giving regret bounds against OPT$(1)$, and we use this result in Section \ref{sec:application} to get regret bounds against $\text{OPT}(B')$ for $B'<B$ in OSFM.  Decoupling the regret benchmark $\text{OPT}(B')$ from the algorithm budget $B$ can be useful when one would like to control the strength of a regret bound against a specific target $\text{OPT}(B')$ for theoretical or applied reasons. For example \cite{DBLP:journals/toc/AroraHK12} survey a wide variety of applications which rely on bounding $R_{T}^{*}$, but bounds such as $\mathop{\mathbb{E}}[(1-e^{-1})\text{OPT}(B)-\text{Reward}_{T}] \leq \mathcal{O}(\sqrt{TB\ln(N)})$ do not immediately imply useful bounds on $\text{OPT}(1)-\text{Reward}_{T}$.

\section{Follow the Perturbed Multiple Leaders}

We begin by considering the full feedback setting. We first check that allowing the algorithm to choose $B>1$ arms per round, while only competing against the best \textit{single} fixed arm in hindsight, does not make the problem trivial. We do this by showing that any deterministic algorithm with budget $B<N$ still achieves linear regret in the number of rounds. This is achieved by setting $c_{t}(a)=1$ if $a \in S_t$, $c_{t}(a)=0$ otherwise.

\begin{proposition}\label{prop:worst-case-deterministic}
In the full feedback setting, any deterministic algorithm with arm budget $B \leq N$ per round has worst-case regret $R_{T}^{*} \geq \left(1-\frac{B}{N}\right)T$.
\end{proposition}

Likewise, it can be shown that the algorithm which chooses a uniformly random subset of $B$ arms in each round has worst-case expected regret at least $T(1-\frac{B}{N})^{B}$ (achieved by having one arm have cost $0$ across all rounds and every other arm having cost $1$). These two observations show that any solution for achieving sub-linear regret in $T$ requires randomization which depends in some non-trivial way on the prior observed costs even when $B>1$.

\subsection{Generalizing Follow the Perturbed Leader}
\label{sec:FPML}

Choosing the current lowest perturbed-cost arm in each round, \textit{Follow the Perturbed Leader} (\textbf{FPL}) \citep{DBLP:journals/jcss/KalaiV05}, is a well-known regret minimization technique which achieves optimal worst-case regret against adaptive adversaries in the OLwE setting. In this section we generalize the \textbf{FPL} algorithm to \textit{Follow the Perturbed Multiple Leaders} (\textbf{FPML}). In each round, \textbf{FPML} perturbs the cumulative costs of each arm by adding noise, and then picks the $B$ arms with lowest cumulative perturbed cost. This is precisely \textbf{FPL} when $B=1$. We show how one can extend the proof techniques of \citet{DBLP:journals/jcss/KalaiV05} in a natural way to prove that \textbf{FPML} achieves worst-case regret $R_{T}^{*}(\text{\textbf{FPML}}) \leq 2T^{\frac{1}{B+1}}(1+\ln(N))^{\frac{B}{B+1}}$.

\begin{algorithm}\label{alg:fpl-vs-1}
\caption{\textbf{FPML}($B$,$\epsilon$)}
\begin{algorithmic}
\Require $N \geq B \geq 1, \epsilon>0$.

\State Initialize the cumulative cost $C_{0}(a)\leftarrow 0$ for each arm $a \in \mathcal{A}$.

\For{round $t=1,\dots,T$}

\State 1. For each arm $a \in \mathcal{A}$, draw a noise perturbation $p_{t}(a) \sim \frac{1}{\epsilon}$ Exp.

\State 2. Calculate the perturbed cumulative costs for round $t-1$, $\tilde{C}_{t-1}(a) \leftarrow C_{t-1}(a)-p_{t}(a)$.

\State 3. Pull the $B$ arms with the lowest perturbed cumulative costs according to $\tilde{C}_{t-1}$.  Break ties arbitrarily. 

\State 4. Update the cumulative costs for each arm, $C_{t}(a) \leftarrow C_{t-1}(a)+ c_{t}(a)$.

\EndFor

\end{algorithmic}
\end{algorithm}

\begin{theorem}\label{thm:FPML-upper-bound} In the full feedback setting, where $S_t \subset \mathcal{A}$ is the subset of arms chosen by \textbf{FPML} in round $t$, we have:
\[
    \max_{c_1,\dots,c_{T}}\mathop{\mathbb{E}}\left[(1-\epsilon^{B})\sum_{t=1}^{T} c_{t}(S_t)-\min_{a^{*} \in \mathcal{A}} \sum_{t=1}^Tc_{t}(a^{*})\right] \leq \frac{(1+\ln(N))}{\epsilon}.
\]
In particular, for $\epsilon=((\ln(N)+1)/T)^{1/(B+1)}$, we have
\[
 R_{T}^{*}(\text{\textbf{FPML}}) \leq 2T^{\frac{1}{B+1}}(1+\ln(N))^{\frac{B}{B+1}}.
\]
\end{theorem}

The proof follows the same three high level steps which appear in \citet{DBLP:journals/jcss/KalaiV05} for \textbf{FPL}, but we extend these ideas to the case where $B>1$. We first observe that the algorithm which picks the $B$ lowest cumulative cost arms in each round only incurs regret when the best arm in round $t$ is not one of the best $B$ arms in round $t-1$.

\begin{lemma}\label{lem:FPML-overtaking-bound}
 Consider a fixed sequence of cost functions $c_{1},\dots,c_{T}$. 
Let $a_{t}^{*,j}$ be the $j$\textsuperscript{th} lowest cumulative cost arm in hindsight after the first $t$ rounds, breaking ties arbitrarily. Let $S_{t}^{*}\coloneqq\{a_{t-1}^{*,j}\mid j \in [B]\}$ be the set of the $B$ lowest cost arms at the end of round $t-1$. Then for each \(i \in [T]\),
\[
    R_{i}:=\sum_{t=1}^i c_{t}(S_t^{*})-\min_{a^{*} \in \mathcal{A}} \sum_{t=1}^i c_{t}(a^{*}) \leq \sum_{t=1}^{i} \mathds{1}[a^{*,1}_{t} \not \in S_{t}^{*}]
\]
and $R_{i}-R_{i-1} \leq \mathds{1}[a^{*,1}_{i} \not \in S_{i}^{*}]$.
\end{lemma}

This is a generalization of the familiar result that when $B=1$, following the leader has regret bounded by the number of times the leader is overtaken \citep{DBLP:journals/jcss/KalaiV05}. 

The second step is to argue that if the cumulative costs are perturbed slightly, it becomes unlikely that the event $\{a^{*,1}_{t} \not \in S_{t}^{*}\}$ will occur. One way to see this is as follows: fix a round $t$, and let $C_{t-1}(a)$ be the cumulative cost of $a$ at the end of round $t-1$. Let $M\coloneqq C_{t-1}(a^{*,B+1})$. Then every $a \in S_t^*$ has $C_{t-1}(a) \leq M$. If it is also true that  $C_{t-1}(a) < M-c_t(a)$ for any $a \in S_t^*$ then the event  $\{a^{*,1}_{t} \not \in S_{t}^{*}\}$ cannot occur. This is because $C_{t}(a)<(M-c_t(a))+c_{t}(a)=M$ but any $a' \in \mathcal{A}-S_t$ has $C_{t}(a')\geq M$, so $a' \not = a^{*,1}_{t}$. If we had initially perturbed each $C_{t-1}(a)$ by subtracting independent exponential noise $p(a)\sim \frac{1}{\epsilon} \text{Exp}$, then conditional on $M$ the event $\{C_{t-1}(a) < M-c_t(a)\}$ is jointly independent for each $a \in S_t^*$. Moreover the probability of this inequality not holding is equal to $\mathds{P}[p(a)<v+c_{t}(a)|p(a) \geq v]$ for $v$ equal to the unperturbed cost of $a$ at round $t-1$ minus $M$, which is bounded by $\epsilon c_{t}(a)$ (due to the memorylessness property of the exponential distribution).

\begin{lemma}\label{lem:FPML-perturb}
Fix a sequence of cost functions $c_{1},\dots,c_{T}$. Let $C_{i}(a)=\sum_{t=1}^{i} c_{t}(a)$ and $\tilde{C}_{i}(a)=C_{i}(a)-p(a)$ be the perturbed cumulative cost of arm $a$ at the end of round $i$, where $p(a) \sim \frac{1}{\epsilon} \text{Exp}$. Let $\tilde{a}_{t}^{*,j}$ be the $j$\textsuperscript{th} lowest cumulative cost arm in hindsight after the first $t$ rounds using these perturbed costs, and let $\tilde{S}_{t}^{*}=\{\tilde{a}_{t-1}^{*,j}\mid j \in [B]\}$. Then 
\[
    \mathop{\mathbb{E}}\left[\mathds{1}\left[\tilde{a}^{*,1}_{t} \not \in \tilde{S}_{t}^{*}\right]\right] \leq \mathop{\mathbb{E}}\left[\epsilon^{B} c_{t}(\tilde{S}_{t}^{*})\right].
\]
\end{lemma}
Again, when $B=1$ this argument and bound coincides with the argument given by \cite{DBLP:journals/jcss/KalaiV05}.

The final step is to combine Lemmas \ref{lem:FPML-overtaking-bound} and \ref{lem:FPML-perturb} to argue that \textbf{FPML} achieves expected regret at most $\mathop{\mathbb{E}}\left[\epsilon^{B} \sum_{t=1}^{T} c_{t}(\tilde{S}^{*}_t)\right]$ with respect to the perturbed cumulative cost $\tilde{C}_T$. Since $\max_{a \in \mathcal{A}} \mathop{\mathbb{E}}[p(a)] \leq \frac{1+\ln(N)}{\epsilon}$ we can argue we also achieve low expected regret with respect to the unperturbed cost $C_T$. In the setting of this paper, drawing new random perturbations $p_{t}(a)$ in each round is not strictly necessary (we can take $p_{t}(a)=p_{1}(a)$ for $t>1$), but it is necessary to achieve regret bounds when cost functions can depend on prior arm choices of the algorithm (the \textit{adaptive} adversarial setting).  In the setting of this paper where the costs are fixed, the expected regret in either case is the same.

\paragraph{Probabilistic guarantees:} One advantage of this proof technique is that the regret is bounded using the positive random variable $\sum_{t=1}^{T} \mathds{1}\left[\tilde{a}^{*,1}_{t} \not \in \tilde{S}_{t}^{*}\right]$. This means that one can apply methods like Markov inequality to give a probabilistic guarantee of small regret, which is substantially stronger than saying the regret is small in expectation.

\paragraph{Comments on settings of parameters:} When $B=1$ we recover the standard $\mathcal{O}(\sqrt{T\ln(N)})$ regret bound for the OLwE problem. For $B>1$, the regret growth rate as a function of the number of rounds is $T^{\frac{1}{B+1}}$. In particular, when $B=\Omega(\ln(T))$ grows slowly with the number of rounds, the expected regret becomes $\mathcal{O}(\ln(N))$ and does not grow with the number of rounds $T$. If we use a tighter inequality in the proof of Lemma \ref{lem:FPML-perturb}, it is possible to get constant expected regret when $B=\ln(T)\ln(A)$ grows slowly with the number of arms and rounds. 

\paragraph{Lower bounds:} A standard technique for constructing lower bounds in the online experts setting with $B=1$ is to consider costs which are i.i.d. Bernoulli$(p)$ \citep{lattimore2020bandit}. Unfortunately this technique fails when $B>1$ because the expected cost of the minimum of $B$ i.i.d. Bernoulli random variables is generally smaller than the expected cost of the best arm in hindsight unless $p$ is very close to $1$. We are able to show very weak lower bounds in the full feedback setting of $R_{T}^{*}=\Omega(\ln(N))$ for constant $B$ and $T= \Omega(\ln(N))$, but there is a substantial gap between this and the upper bound. We think constructing stronger lower bounds is an interesting problem for future work which may require new analysis techniques.  

\paragraph{Partial feedback:} The \textit{semi-bandit feedback} setting is a form of partial feedback where the algorithm only observes the individual costs of the arms it pulls. It can be shown that passing unbiased cost function estimates to \textbf{FPML} results in a similar regret bound in the semi-bandit feedback setting; the result is specific to \textbf{FPML} and using unbiased cost functions does not generally work for any $R^{*}_{T}$-minimizing algorithm when $B>1$ because of the non-linearity of the cost function. In the case of \textbf{FPML} this is not an issue because the same bounding technique using Lemma \ref{lem:FPML-overtaking-bound} holds in expectation when using unbiased cost estimators. A naïve way to generate unbiased cost estimates in this setting is to use an additional arm to uniformly sample costs; in Section \ref{sec:experiments} we explore geometric sampling \citep{neu2013efficient} for getting unbiased cost estimates which is effective in practice.

\begin{proposition}\label{prop:full-to-partial}
Define the algorithm \textbf{FPML-partial} which simulates \textbf{FPML}, passing it unbiased cost estimates $\hat{c}_{t} \in [0,K]$ at round $t \in [T]$,\footnote{More formally: $\mathop{\mathbb{E}}[\hat{c}_{t}(\cdot) \mid \mathcal{F}_{t-1}]=c_{t}(\cdot)$ where $\mathcal{F}_{t-1}$ is the $\sigma$-algebra generated by all the randomness up to and including round $t-1$.} and copies the arm choices of \textbf{FPML}. Then we have $R_{T}^{*}(\text{\textbf{FPML-partial}})\leq \ln(N)/\epsilon + T (1-\e^{-K\epsilon})^B$. For $\epsilon = \left(\frac{\ln(N)}{TK^B}\right)^\frac{1}{B+1}$, $R_{T}^{*}(\text{\textbf{FPML-partial}})\leq \mathcal{O}( T^{\frac{1}{B+1}} (K\ln(N))^{\frac{B}{B+1}})$.
\end{proposition}

\section{Generalized regret bounds for Online Submodular Function Maximization}
\label{sec:application}

\citet{DBLP:conf/nips/StreeterG08} considered the more general problem of Online Submodular Function Maximization which captures a number of previously studied problems as special cases. For the sake of emphasizing the key ideas, we restrict attention to the full feedback setting where each action has unit duration. The OSFM problem in this setting is as follows:

\begin{dfn}
    Define a \emph{schedule} to be a finite sequence of actions\footnote{In their original problem definition actions may each have a different associated \textit{duration}.} \(a \in \mathcal{A}\), and let \(\mathcal{S}\) be the set of all schedules; the \emph{length} \(\ell(S)\) of a schedule \(S \in \mathcal{S}\) is the number of actions it contains. Define a \emph{job} to be a function \(f : \mathcal{S} \to [0,1]\) such that for any schedules \(S_1,S_2 \in \mathcal{S}\) and any action \(a \in \mathcal{A}\):
    \begin{enumerate}
        \item \(f(S_1) \le f(S_1 \oplus S_2)\) and \(f(S_2) \le f(S_1 \oplus S_2)\) \textbf{(monotonicity)};
        \item \(f(S_1 \oplus S_2 \oplus \langle a \rangle) - f(S_1 \oplus S_2) \le f(S_1 \oplus \langle a \rangle) - f(S_1)\) \textbf{(submodularity)}.
    \end{enumerate}
\end{dfn}

\begin{dfn}[Online Submodular Function Maximization]
    The problem consists of a game with \(T\) rounds. We are given some fixed \(B > 0\) and at each round \(t \in [T]\) we must choose a schedule \(S_t \in \mathcal{S}\) with \(\ell(S_t) \le B\) to be evaluated by a job \(f_t\) which is only revealed after our choice. The goal is to maximize the cumulative output $\text{Reward}_{T}:=\sum_{t=1}^T f_{t}(S_t)$.
\end{dfn}

\citet{DBLP:conf/nips/StreeterG08} propose an online greedy algorithm \textbf{OG} which achieves the guarantee $(1-e^{-1})\text{OPT}(B)-\text{Reward}_{T} \leq \mathcal{O}(\sqrt{TB \ln (N)})$ in expectation, where $\text{OPT}(B)$ is the cumulative reward of the best fixed schedule of length $B$ in hindsight. In this section we explain how to use \textbf{FPML} to extend their algorithm to allow a trade-off between budget resources and regret bounds.

The algorithm \textbf{OG} has two key ideas. Suppose that we start with an empty schedule $S_{0}\coloneqq\emptyset$. The first idea is that if we knew $f_{1},\dots,f_{T}$ in advance, we could greedily construct $S_{i}\coloneqq S_{i-1} \oplus \langle\arg \max_{a_i \in \mathcal{A}} \sum_{t=1}^T f_t(S_{i-1} \oplus \langle a_i\rangle) -f_t(S_{i-1})\rangle$ for $i=1,\dots,B$, where $a_i$ is the best greedy arm in hindsight for greedy round $i$. It can be shown that submodularity then implies $(1-e^{-1})\text{OPT}(B) \leq \sum_{t=1}^T f(S_B)$. Since we don't know $f_1,\dots,f_T$ in advance, the second idea is to run $B$ copies of a $R_{T}^{*}$ regret minimizer, where the $i$\textsuperscript{th} copy tries to compete with achieving the same \textit{improvement} of cumulative reward as the best fixed greedy action $a_i$ in hindsight. The regret bound on the $i$\textsuperscript{th} copy with respect to the best greedy arm in hindsight is $\mathcal{O}(\sqrt{T\ln N})$; across the $B$ copies one can show the net regret compared to the offline greedy solution is bounded by $\mathcal{O}(\sqrt{TB\ln N})$, which is where the final bound comes from. In summary, \textbf{OG} works as follows: for each round $t \in [T]$, run $B$ greedy rounds. In greedy round $i=1,\dots,B$, pull the arm $a^{t}_{i}$ proposed by the $i$\textsuperscript{th} black box $R_{T}^{*}$-minimizer, set $S_{t,i}\coloneqq S_{t,i-1}\oplus \langle a^{t}_{i} \rangle$, and feed back the greedy rewards $r_{t,i}(a)=f_{t}(S_{t,i-1}\oplus \langle a\rangle)-f_{t}(S_{t,i-1})$ to the $i$\textsuperscript{th} black box.

We propose a hybrid version of \textbf{OG}, called \textbf{OG\textsubscript{hybrid}}, which for any budget $B$ allows us to compete asymptotically well against OPT($B'$) for any chosen $B' \leq B/\ln(T)$. The algorithm \textbf{OG\textsubscript{hybrid}} is based on the following two changes to \textbf{OG}:
\begin{enumerate}
    \item Instead of having $B$ greedy rounds, we have  $B'\ln(T)$ greedy rounds. One can show that the extra factor of $\ln(T)$ allows one to drop the $(1-e^{-1})$ term in the regret bound.
    \item Instead of running a one-arm-pulling $R_{T}^{*}$ minimizer in each greedy round, we run \textbf{FPML} which pulls $\lfloor B/B'\ln(T)\rfloor$ arms. This allows us to improve the regret bound for each of the $B'\ln(T)$ $R_{T}^{*}$-minimizers and directly translates to a tighter overall regret bound.
\end{enumerate}
Besides these two changes, the algorithm is identical to that in \cite{DBLP:conf/nips/StreeterG08}.

More generally, let \textbf{OG\textsubscript{hybrid}}$(B,\tilde{B})$ denote the algorithm where each \textbf{FPML} box has a budget of $\tilde{B}$ and there are $B/\tilde{B}$ greedy rounds, so that the total number of arms pulled in each round is $B$. This algorithm is a hybrid of \textbf{OG} and \textbf{FPML} in the sense that \textbf{OG\textsubscript{hybrid}}$(B,1)$ is \textbf{OG} with \textit{Follow the Perturbed Leader} as the $R_{T}^{*}$-minimizing subroutine. On the other hand, \textbf{OG\textsubscript{hybrid}}$(B,B)$ is \textbf{FPML}. Varying $\tilde{B}$ allows us to interpolate between these two algorithms by varying the budget we give to the \textbf{FPML} subroutines. 

On a technical note, because we are pulling $\tilde{B} \geq 1$ arms for each greedy choice, we require a slight strengthening on the monotonicity condition which is common to many practical applications including the experiments we consider in the next section:

\begin{assumption} \label{additional_assumption}
    In addition to monotonicity and submodularity, each job \(f : \mathcal{S} \to [0,1]\) also satisfies \(f(S_1 \oplus S_2 \oplus S_3) \ge f(S_1 \oplus S_3)\) for any schedules \(S_1,S_2,S_3 \in \mathcal{S}\).
\end{assumption}

We can then give the following bounds:

\begin{theorem} \label{thm:streeter_generalized}
    For any choices of $B',\tilde{B}$, under \cref{additional_assumption} and with budget \(B \coloneqq \lceil B' \widetilde{B} \ln(T)\rceil \), algorithm \textbf{OG\textsubscript{hybrid}}(B,$\tilde{B}$) experiences expected regret
    \[\mathcal{O}\left(B' \ln(T) T^{1/(\widetilde{B}+1)} \ln (N)^{\widetilde{B}/(\widetilde{B}+1)}\right)\]
    relative to the best-in-hindsight fixed schedule of length \(B'\). In particular, if \(\widetilde{B} = \lceil\ln(T)\rceil\) (and so \(B = B' \lceil\ln(T)\rceil^2)\) the regret is bounded by \(\mathcal{O}(B'\ln(T) \ln(N) )\).
\end{theorem}

This bound allows us the flexibility of trading off a schedule budget $B$ for how tightly we would like to compete with the best fixed schedule of length $B' \leq B$ in hindsight. 

\paragraph{Partial feedback:} \citet{DBLP:conf/nips/StreeterG08} extend their algorithm to handle partial feedback by replacing the $R_{T}^{*}$ regret minimizers with the bandit algorithm \textbf{Exp3} \citep{auer2002nonstochastic} which only requires feedback on the arms which are pulled. Likewise, one can replace \textbf{FPML} with \textbf{FPML-partial} in \textbf{OG\textsubscript{hybrid}} to get an algorithm which gives regret bounds in a semi-bandit feedback setting (where we can observe $f_t(S)$ for any schedule $S$ consisting of actions which were pulled in round $t$). We empirically compare the partial feedback versions of \textbf{OG} and \textbf{OG\textsubscript{hybrid}} in the next section.

\section{Experiments: online hyperparameter optimization}\label{sec:experiments}

The problem of \textit{black-box optimization}---where a hidden function is to be minimized using as few evaluations as possible---has recently generated increased interest in the context of hyperparameter selection in deep learning \citep{snoek2012practical,liu2020towards,bouthillier2020survey}. As a result, the 2020 NeurIPS BBO Challenge \citep{turner2021bayesian} invited participants' optimizers to compete to find the best possible configurations of several ML models on a number of common datasets, given a limited budget of training cycles for each. One of the key findings was that sophisticated new algorithms are normally outperformed on average by techniques that ensemble existing methods \citep{liu2020gpu}, i.e. optimizers tend to have varying strengths and weaknesses that are suited to different task types. In a scenario where many hyperparameter selection problems are to be processed (e.g. in a data center) and limited computing resources are available, it may thus be desirable to learn over time how best to choose \(B\) optimizers to apply independently to each problem (e.g. to run in parallel on \(B\) available CPU cores). This is a natural partial feedback application of bandit algorithms that pull multiple arms per round and receive the best score found across the optimizers which are run. 

\subsection{Experimental setup}

We follow a similar approach to the NeurIPS BBO Challenge; the optimization problem at each round \(t \in [T]\) is to choose the hyperparameters of either a multi-layer perceptron (MLP) or a lasso classifier for one of 184 classification tasks from the Pembroke Machine Learning Benchmark \citep{olson2017pmlb} (so \(T=368\)). At each round the bandit algorithm must select \(B\) Bayesian black-box optimizers from a choice of 9 to run in parallel on the current problem; the received \textit{reward}\footnote{Here we use rewards to be consistent with the setting of Online Submodular Function Maximization.} for each optimizer was calculated as a \([0,1]\)-normalized measure of where the best training loss attained sits between \textbf{(a)} the expected best loss from a random hyperparameter search and \textbf{(b)} an estimate of the best possible loss attainable (the same approach used by the Bayesmark package \citep{bayesmark}). Rewards are only observed for the optimizers which are run (semi-bandit feedback). For each budget level $B=1,\dots,6$, we ran a benchmarking study comparing the performance of the following bandit algorithms in this setting: (1) \textbf{FPML-partial}$(B)$ from Section \ref{sec:FPML}, using geometric sampling   \citep{neu2013efficient} to construct unbiased cost estimates (see appendix for details). (2) \textbf{OG\textsubscript{hybrid}}$(B,\tilde{B})$ from Section \ref{sec:application}, with \textbf{FPML-partial} as a subroutine. We do this for varying values of $\tilde{B}$ to see how performance changes. (3) \textbf{OG}$(B)$, the partial feedback version of the original online greedy algorithm from \cite{DBLP:conf/nips/StreeterG08}. We benchmark performance against \textbf{BIH}$(B)$, the score of the Best fixed subset of size $B$ In Hindsight. In all cases the $\epsilon$ parameter for the bandit subroutines \textbf{FPML-partial} and \textbf{Exp3} was set to their theoretically optimal values given $B$, $N$ and $T$ without any fine-tuning (for \textbf{FPML-partial} we use Proposition \ref{prop:full-to-partial}). We run each algorithm setting 100 times to estimate the mean and standard deviation of the performance.

\subsection{Results and discussion}

\renewcommand{\c}[1]{#1}
\renewcommand{\d}[1]{#1}
\begin{table}
  \caption{Experimental results for $B=6$ and $B=4$.}
  \label{table:experimental results}
  \centering
  \begin{tabular}[t]{lll}
    \toprule
    \multicolumn{3}{c}{$B=6$}                   \\
    \cmidrule(r){1-3}
    Algorithm     & Mean Reward    & StD \\
    \midrule
    \textbf{BIH}$(6)$ & \c{0.901} & \d{NA} \\
        \textbf{FPML-partial}$(6)$ & \c{0.888} & \d{0.0072} \\
        \textbf{OG\textsubscript{hybrid}}$(6,3)$ & \c{0.836} & \d{0.0111} \\
        \textbf{OG\textsubscript{hybrid}}$(6,2)$ & \c{0.814} & \d{0.0143} \\
        \textbf{OG\textsubscript{hybrid}}$(6,1)$ & \c{0.785} & \d{0.0137} \\
        \textbf{OG} & \c{0.767} & \d{0.0157} \\
    \bottomrule
  \end{tabular}
  \quad
  \begin{tabular}[t]{lll}
    \toprule
    \multicolumn{3}{c}{$B=4$}                   \\
    \cmidrule(r){1-3}
    Algorithm     & Mean Reward    & StD \\
    \midrule
    \textbf{BIH}$(4)$ & \c{0.836} & \d{NA} \\
    \textbf{FPML-partial}$(4)$ & \c{0.813} & \d{0.0108} \\
    \textbf{OG\textsubscript{hybrid}}$(4,2)$ & \c{0.756} & \d{0.0149} \\
        \textbf{OG\textsubscript{hybrid}}$(4,1)$ & \c{0.716} & \d{0.0178} \\
        \textbf{OG}$(4)$ & \c{0.689} & \d{0.0151} \\
    \bottomrule
  \end{tabular}
  
\end{table}

 As expected, \textbf{OG\textsubscript{hybrid}}$(B,1)$ has very similar performance to \textbf{OG}$(B)$ in all cases because they implement essentially the same algorithm (the difference being due to different choices of one-arm pulling $R_{T}^{*}$ minimizers, \textbf{FPML-partial}$(1)$ and \textbf{Exp3}). However, we notice that in every instance, allocating more budget to the \textbf{FPML-partial} subroutine in \textbf{OG\textsubscript{hybrid}} (increasing $\tilde{B}$) while keeping the overall budget constant improved the average score performance. This means that \textbf{OG\textsubscript{hybrid}} was always at least as good as \textbf{OG} for all parameter settings, and that \textbf{FPML-partial} outperformed both of these algorithms in all cases. This is perhaps surprising, because \textbf{OG} and \textbf{OG\textsubscript{hybrid}} are designed to achieve low regret against the stronger benchmark of the best subset of $B'\geq 1$ arms in hindsight, while \textbf{FPML-partial} is only designed to achieve low regret with respect to the single best optimizer in hindsight. Towards explaining this observation, we notice that for this particular dataset, the best subset of \(B\) arms in hindsight happens to be very similar to the set of the individual best performing \(B\) arms in hindsight, and \textbf{FPML} achieves low regret with with respect to the latter (see \cref{prop:top_B_regret} below). This raises the interesting question of when certain regret objectives might be better than others in practice. \\

\begin{proposition} \label{prop:top_B_regret}
In the full feedback setting, the expected regret of \textbf{FPML} with \(B\) arms relative to the set of the individual best performing \(B\) arms in hindsight is at most \([1-B^{-1}+\ln(|\mathcal{A}|/B)]/\epsilon + T(1-\exp(-\epsilon B)).\)
\end{proposition}

\paragraph{Synthetic tasks:} We also evaluated these partial feedback algorithms in a number of synthetic environments, exploring examples where \textbf{(a)} the optimal subset of \(B\) arms, \textbf{(b)} the subset of \(B\) arms chosen greedily, \textbf{(c)} the individual best performing \(B\) arms, perform in various ways relative to each other; \textbf{OG} approximates \textbf{(b)} and \textbf{FPML-partial} approximates \textbf{(c)}, so the closeness of either of these to \textbf{(a)} determines each algorithm's performance. We find that problem instances exist where \textbf{OG} outperforms \textbf{FPML-partial} and vice versa. See the appendix for details.

\section{Connections to other problems}\label{sec:lp}

\cite{DBLP:journals/toc/AroraHK12} surveys a wide variety of algorithmic problems and shows how they can all be solved using a $R_{T}^{*}$ minimizing subroutine in the standard OLwE setting with $B=1$. The purpose of this section is to highlight the relative ease with which algorithms like \textbf{FPML} can sometimes be \textit{plugged in} to existing algorithms which use an $R_{T}^{*}$ minimizer as black box subroutine, with little or no alteration to their proofs of correctness. We already saw an example of this in Section \ref{sec:application}, where the regret minimizing subroutine was replaced with \textbf{FPML}. The resulting algorithm gave improved regret guarantees at the cost of higher budget requirements, without changing the structure of the original proof given in \cite{DBLP:conf/nips/StreeterG08}. In this section we take Linear Programming as an example, and illustrate what its plug-in algorithm looks like. We are unsure whether the resulting algorithm is necessarily useful because it requires a more powerful oracle than the one supposed in \cite{DBLP:journals/toc/AroraHK12}; but we do think that exploring the algorithms which result from this plug-in technique more broadly may be an area of interest for future work. 

\subsection{Linear Programming}

We consider the Linear Programming (LP) problem from \cite{DBLP:journals/toc/AroraHK12}: Given a convex set $P$, an $n \times m$ matrix A with entries in $\mathbb{R}$, and a vector $b \in \mathbb{R}^n$, the task is to find an $x \in P$ such that $Ax \geq b$, or determine that no such $x$ exists. We assume  that we have an oracle which solves the following easier problem (where $A_i$ denotes the $i$\textsuperscript{th} row of $A$): 
\begin{dfn}
    A $(\rho,B)$-bounded oracle for $\rho \geq 0$ is an algorithm, which when given a joint distribution $d$ over $[n]^B$, finds an $x \in P$ such that
    \[
        \mathop{\mathbb{E}}_{(i_1,\dots,i_B) \sim d}\left[\min_{i \in \{i_1,\dots,i_B\}} A_{i}x-b_{i}\right] \geq 0
    \]
    or determines that no such $x$ exists. If an $x$ is found, then $\forall i \in [n], |A_{i} x-b_i | \leq \rho$. 
\end{dfn}

When $B=1$ this is a simplified version of the oracle defined in \cite{DBLP:journals/toc/AroraHK12}, and it is equivalent to an oracle which can find $x \in P$ which satisfies a single linear constraint $d^{\top} A x \geq d^{\top} b$ (when viewing $d$ as a vector in $\mathbb{R}^n$). When $B=n$, solving this problem can be as hard as solving the original LP problem by jointly choosing $i_j$ to have probability mass 1 on the $j$\textsuperscript{th} linear constraint. $1 < B < n$ represents an intermediate regime where the oracle needs to find an $x$ which `fools' the joint distribution $d$. For simplicity, we assume that we have access to an oracle which takes as input bounded cost functions $c_1,\dots,c_{t-1}$, and outputs the joint distribution $d_t$ over arms of \textbf{FPML} in round $t$ after observing cost functions $c_1,\dots,c_{t-1}$. In practice such an oracle could be achieved by e.g. sampling arm draws from \textbf{FPML} to approximate $d_t$. 

\begin{proposition}\label{prop:lp}
Let $\epsilon>0$. Suppose there exists a $(\rho,B)$-bounded oracle for the feasibility problem $\exists x \in P \text{ s.t. } Ax \geq b$. Then there is an algorithm which either finds an $x \in P$ s.t. $\forall i \in [n], A_ix \geq b_i -\epsilon$, or correctly concludes that the problem is infeasible. The algorithm makes at most $T=\left(\frac{1}{\epsilon}\right)^{\frac{B+1}{B}} (4\rho)^{\frac{B+1}{B}} (1+\ln(n))$ calls to the $(\rho,B)$-bounded oracle and \textbf{FPML} oracle, with a total runtime of $\mathcal{O}(T)$. 
\end{proposition}

The proof technique for general $B \geq 1$ is essentially identical to the proof given in \cite{DBLP:journals/toc/AroraHK12} for $B=1$ except that we are able to use a stronger bound on $R_{T}^{*}$ (see appendix for details).  

\paragraph{Comment on bound:} When $B=1$, we require $\Omega((\frac{1}{\epsilon})^2)$ steps in order to find an $x$ which is $\epsilon$-close to satisfying the constraint. The $(\frac{1}{\epsilon})^2$ term comes from the fact that when $B=1$, the average regret for OLwE is $\Omega(\sqrt{T}/T)=T^{-\frac{1}{2}}$ , so it takes $T \geq (\frac{1}{\epsilon})^2$ steps for the average regret to be $\leq \epsilon$. The quadratic dependence on an accuracy parameter is therefore common in many applications which use OLwE with $B=1$ as a subroutine (including Boosting \citep{DBLP:journals/ml/Schapire90} and solving zero sum games \citep{freund1999adaptive}). For general $B\geq 1$, we only require $\mathcal{O}((\frac{1}{\epsilon})^{\frac{B+1}{B}})$ steps (suppressing terms related to $n$ and $\rho$) for the average regret to be $\leq \epsilon$. In the case of Linear Programming, this is at the expense of requiring a stronger oracle for the problem.

\section{Conclusion}

This paper presented a new algorithm, Follow the Perturbed Multiple Leaders, which allows one to directly trade off budget constraints for bounds on regret. We showed how \textbf{FPML} can be used as a subroutine to generate new algorithms for Online Submodular Function Optimization and Linear Programming which trade off resources and oracle power for improved performance guarantees. 

We also highlight a number of areas for future work: (a) \textbf{Lower bounds:} Can we reduce the gap between the upper and lower bounds of this problem? Improving lower bounds may require new techniques than are traditionally used in the OLwE setting. (b) \textbf{Plug-in algorithms:} Are there other cases where using \textbf{FPML} in existing algorithms can lead to new theoretical results? And when are these new algorithms practically useful? (c) \textbf{Which regret benchmarks are useful in practice:} The experiments of Section \ref{sec:experiments} showed that algorithms designed to minimize regret with respect to a single arm can sometimes in practice outperform algorithms designed to minimize regret with respect to the stronger benchmark of the best subset of arms. Are certain regret objectives better than others in different practical applications?

\bibliography{references}

\newpage

\section*{Checklist}

\begin{enumerate}

\item For all authors...
\begin{enumerate}
  \item Do the main claims made in the abstract and introduction accurately reflect the paper's contributions and scope?
    \answerYes{}
  \item Did you describe the limitations of your work?
    \answerYes{Section \ref{sec:experiments} discusses when the new algorithms perform better or worse relative to existing work. Section \ref{sec:lp} acknowledges that plug-in reductions with FPML do not necessarily result in new algorithms which are useful in practice.} 
  \item Did you discuss any potential negative societal impacts of your work?
    \answerNA{The contributions are primarily theoretical/related to optimization in general.}
  \item Have you read the ethics review guidelines and ensured that your paper conforms to them?
    \answerYes{}
\end{enumerate}

\item If you are including theoretical results...
\begin{enumerate}
  \item Did you state the full set of assumptions of all theoretical results?
    \answerYes{The setting and assumptions are fully defined in the main body of the paper.}
        \item Did you include complete proofs of all theoretical results?
    \answerYes{Every proposition, lemma and theorem in the main body of the paper has a proof in the appendix. The main body of the paper is primarily used for communicating the high level intuition behind the proofs.}
\end{enumerate}

\item If you ran experiments...
\begin{enumerate}
  \item Did you include the code, data, and instructions needed to reproduce the main experimental results (either in the supplemental material or as a URL)?
    \answerYes{Included in supplementary material. We will also provide a github link after review.}
  \item Did you specify all the training details (e.g., data splits, hyperparameters, how they were chosen)?
    \answerYes{All choices of algorithm parameters are explained in the main body of the paper. Further details can be found in the appendix.}
        \item Did you report error bars (e.g., with respect to the random seed after running experiments multiple times)?
    \answerYes{Each algorithm variant in Section \ref{sec:experiments} was run multiple times and standard deviations of results are included.}
        \item Did you include the total amount of compute and the type of resources used (e.g., type of GPUs, internal cluster, or cloud provider)?
    \answerNA{This paper is about benchmarking algorithm performance on idealized standardized budgets. Actual compute time is application specific and not central to the contributions of the paper.}
\end{enumerate}

\item If you are using existing assets (e.g., code, data, models) or curating/releasing new assets...
\begin{enumerate}
  \item If your work uses existing assets, did you cite the creators?
    \answerYes{All prior datasets and algorithms are cited.}
  \item Did you mention the license of the assets?
    \answerYes{A permissive license is included with the code in the supplementary material, and will be made available publicly after review.}
  \item Did you include any new assets either in the supplemental material or as a URL?
    \answerYes{Yes, in the supplemental material we provide code for replicating the experiments in Section \ref{sec:experiments}}. 
  \item Did you discuss whether and how consent was obtained from people whose data you're using/curating?
    \answerNA{Consent was not required and the data used is publicly available for use as a benchmarking dataset by the academic community.}
  \item Did you discuss whether the data you are using/curating contains personally identifiable information or offensive content?
    \answerNA{This is not applicable to the dataset considered.}
\end{enumerate}

\item If you used crowdsourcing or conducted research with human subjects...
\begin{enumerate}
  \item Did you include the full text of instructions given to participants and screenshots, if applicable?
    \answerNA{}
  \item Did you describe any potential participant risks, with links to Institutional Review Board (IRB) approvals, if applicable?
    \answerNA{}
  \item Did you include the estimated hourly wage paid to participants and the total amount spent on participant compensation?
    \answerNA{}
\end{enumerate}

\end{enumerate}

\newpage

\appendix

\section{Section 2: Follow the Perturbed Multiple Leaders}

\paragraph{Proof of Proposition \ref{prop:worst-case-deterministic}}

\begin{proof}
Construct a cost sequence inductively for $t=1,2,\dots,T$: Let $S_t \subset \mathcal{A}$ be the deterministic choice of arms the bandit algorithm chooses for round $t$ given the previously chosen cost functions $c_{1},\dots,c_{t-1}$. Now choose $c_{t}(a)=1$ if $a \in S_t$, and $c_{t}(a)=0$ otherwise. Then the bandit algorithm achieves cost $T$. The total cost summed over all arms is $\sum_{t=1}^T |S_{t}|\leq  BT$, so there must exist at least one $a' \in \mathcal{A}$ such that $\sum_{t=1}^{T} c_{t}(a') \leq \frac{BT}{N}$. Thus $R_{T}^{*} \geq T-\frac{BT}{N}=(1-\frac{B}{N})T$. 
\end{proof}

\paragraph{Proof of Lemma \ref{lem:FPML-overtaking-bound}}

\begin{proof}
Let $R_{i}=\sum_{t=1}^{i} c_{t}(S_t^{*})-\min_{a^{*} \in \mathcal{A}} \sum_{t=1}^{i} c_{t}(a^{*})$ be the regret at the end of round $i$. Then the increase in regret in round $i$ is

\begin{align*}
    r_{i}&:=R_{i}-R_{i-1}\\
    &=\sum_{t=1}^{i}\left(c_{t}(S_{t}^{*})-c_{t}(a^{*,1}_{i})\right)-\sum_{t=1}^{i-1}\left(c_{t}(S_{t}^{*})-c_{t}(a^{*,1}_{i-1})\right)\\
    &=c_{i}(S_{i}^{*}) -c_{i}(a^{*,1}_{i}) +\left(\sum_{t=1}^{i-1}c_{t}(a^{*,1}_{i-1})-\sum_{t=1}^{i-1} c_{t}(a^{*,1}_{i})\right) \\
    &\leq c_{i}(S_{i}^{*})-c_{i}(a^{*,1}_{i})\\
    & \leq \mathds{1}[a^{*,1}_{i} \not \in S_{i}^{*}]
\end{align*}

and the result follows by evaluating $\sum_{t=1}^{T} r_{t}$. 
\end{proof}

\begin{lemma}(Proof of independence for Lemma \ref{lem:FPML-perturb}) \label{lem:indep}
 Let $X_{1},\dots,X_{K}$ be jointly independent continuous random variables. Let $i_{1},\dots,i_{k}$ and $v_{i_1},\dots,v_{i_k}$ be the indices and values of the largest $k<K$ random variables, and let $X:=\{X_i|i \not \in \{i_1,\dots,i_k\}$ be the smallest $K-k$ random variables. Then conditional on $(i_j,v_{i_j})_{j \in [k]}$, the values of each $X_i \in X$ are jointly independent. Moreover, the marginal distribution $X_{i}|(i_j,v_{i_j})_{j \in [k]}$ for $i \not \in \{i_1,\dots,i_k\}$ is $X_{i}|X_{i} \leq \min_{j \in [k]} v_{i_j}$.
\end{lemma}

\begin{proof}
Let $M:= \min_{j \in [k]} v_{i_j}$. The conditional joint density function is

\begin{align*}
    f(X_1,\dots,X_K|(i_j,v_{i_j})_{j \in [k]})&\propto f(X \land (i_j,v_{i_j})_{j \in [k]})\\
    &= \prod_{j \in [K]-\{i_1,\dots,i_k\} } f(X_j) \prod_{j \in \{i_1,\dots,i_k\} } f(X_{j}=v_j) \prod_{j \in [K]-\{i_1,\dots,i_k\} }\mathds{1}[X_{j}\leq M]\\
    &\propto \prod_{j \in [K]-\{i_1,\dots,i_k\} }f(X_{j})\mathds{1}[X_{j}\leq M]\\
\end{align*}

i.e. the joint density factorizes for each $X_j$ (which implies joint independence), and marginally the density for $X_j \in X$ is $\propto f(X_{j})\mathds{1}[X_{j}\leq M]$ which gives the required result. 
\end{proof}

\textbf{Proof of Lemma \ref{lem:FPML-perturb}}

\begin{proof}
Fix a round $t$. Consider the jointly independent random variables $X_{a}=\tilde{C}_{t-1}(a)$ for $a \in \mathcal{A}$. Condition on the values and identities of the $N-B$ largest of these random variables, i.e. condition on $E=\{(X_{a^{*,j}_{t-1}},a^{*,j}_{t-1})\}_{j=B+1}^{N}$, and let $M=X_{a^{*,B+1}_{t-1}}$ be the minimum perturbed cost among these non-leading arms. Impose an ordering on $\mathcal{A}$ and let $l_1,\dots,l_B\in\mathcal{A}-\{a^{*,j}\}_{j=B+1}^{N}$ be the remaining arms (the top $B$ leaders) ordered lexicographically (i.e. not necessarily in order of cumulative perturbed cost). Then it can be shown that the distribution of the random variables $X_{l_1},\dots,X_{l_B}$ conditioned on $E$ is jointly independent, and the marginal distribution of $X_{l_j}$ given $E$ is $X_{l_j}|(X_{l_j} \leq M)$ (see Lemma \ref{lem:indep}). Now observe that if $X_{l_j}<M-c_{t}(l_j)$ for any $j \in [B]$, then the event $(\tilde{a}^{*,1}_{t} \not \in \tilde{S}_{t}^{*})$ is impossible. This is because $l_j \in \tilde{S}_{t}^{*}$, but for any $a \not \in S_{t}^{*}$, $\tilde{C}_{t}(a) \geq M$ but $\tilde{C}_{t}(l_j)=X_{l_j}+c_{t}(l_j)<M$ (i.e. $l_j$ cannot be overtaken by any non-top-$B$-leader in round $t$). Therefore we have

\begin{align}
    \mathop{\mathbb{E}}\left[\mathds{1}[\tilde{a}^{*,1}_{t} \not \in \tilde{S}_{t}^{*}]|E\right] & \leq P\left[\bigwedge_{j=1}^B \lnot(X_{l_j}<M-c_{t}(l_j))\;\middle|\; E\right]\\
    &=\prod_{j=1}^B P\left[ \lnot(X_{l_j}<M-c_{t}(l_j))|X_{l_j} \leq M\right] \label{line:independence}\\
    &=\prod_{j=1}^B \left(1-P\left[p(l_j)>C_{t-1}(l_j)+c_{t}(l_j)-M|p(l_j) \geq C_{t-1}(l_j)-M \right]\right)\\
    &\leq \prod_{j=1}^B \left(1-P\left[p(l_j)>c_{t}(l_j)\right]\right) \label{line:memorylessness}\\
    &= \prod_{j=1}^B \left(1-\int_{c_{t}(l_j)}^{\infty}\epsilon e^{-\epsilon x}dx\right)=\prod_{j=1}^B (1-e^{-\epsilon c_{t}(l_j)})\\
    &\leq \prod_{j=1}^B (\epsilon c_{t}(l_j)) \leq \epsilon^{B} c_{t}(\tilde{S}_{t}^{*}) \label{line:expinequality}
\end{align}
(\ref{line:independence}) follows by independence, (\ref{line:memorylessness}) is due to the memorylessness property of the exponential distribution (with equality unless $C_{t-1}(l_j)-M<0$), and (\ref{line:expinequality}) follows because $1-e^{-x} \leq x$ for $x\geq 0$ and $c_{t}(S_{t}^{*}) \geq \prod_{a \in S_{t}^{*}} c_{t}(a)=\prod_{j =1}^{B} c_{t}(l_j)$. The final claim follows by taking the expectation over the conditioned event $E$.
\end{proof}

\paragraph{Proof of Theorem \ref{thm:FPML-upper-bound}}

\begin{proof}
Consider a modified version of FPML where $p_{t}(a)=p_{1}(a)=p(a)$ for all $t>1$ (i.e. we keep the random perturbation fixed across rounds). Then this version of FPML picks the set $\tilde{S}_{t}^{*}$ in round $t$, and the regret can be bounded as

\begin{align*}
    \sum_{t=1}^{T} c_{t}(\tilde{S}_{t}^{*})- \min_{a^{*} \in \mathcal{A}} \sum_{t=1}^{T} c_{t}(a^{*})&= \left(\sum_{t=1}^{T} c_{t}(\tilde{S}_{t}^{*})-\min_{\tilde{a}^{*} \in \mathcal{A}}\sum_{t=1}^{T} c_{t}(\tilde{a}^{*})-p(\tilde{a}^{*})\right)\\ &+\left( \min_{\tilde{a}^{*} \in \mathcal{A}}\sum_{t=1}^{T} c_{t}(\tilde{a}^{*})-p(\tilde{a}^{*})-\min_{a^{*} \in \mathcal{A}} \sum_{t=1}^{T} c_{t}(a^{*})\right)
\end{align*}

The second term is $\leq 0$. The first term can be interpreted as the regret of a modified version of MAB with a $0$\textsuperscript{th} round with cost function $-p$, where we are only allowed to pull arms from round $t=1$. The regret increase incurred in the $0$\textsuperscript{th} round is at most $\max_{a \in \mathcal{A}} p(a)$. For the remaining rounds, we use Lemma \ref{lem:FPML-overtaking-bound} followed by Lemma \ref{lem:FPML-perturb} to get

\begin{align*}
    \mathop{\mathbb{E}}\left[\sum_{t=1}^{T} c_{t}(S_{t}^{*})- \min_{a^{*} \in \mathcal{A}} \sum_{t=1}^{T} c_{t}(a^{*})\right]&\leq \sum_{t=1}^{T}\mathop{\mathbb{E}}\left[\mathds{1}[\tilde{a}^{*,1}_{t} \not \in \tilde{S}_{t}^{*}]\right]+\mathop{\mathbb{E}}\left[\max_{a \in \mathcal{A}} p(\tilde{a})\right]\\
    & \leq \epsilon^T \mathop{\mathbb{E}}\left[\sum_{t=1}^{T} c_{t}(S_{t}^{*})\right]+\frac{(1+\ln(N))}{\epsilon}
\end{align*}

Where the inequality on $\mathop{\mathbb{E}}\left[\max_{a \in \mathcal{A}} p(\tilde{a})\right]$ comes from \citet{DBLP:journals/jcss/KalaiV05}. The final step is to argue that the unmodified version of FPML which chooses independent noise $p_{t}(a)$ in each round also achieves this bound. This is immediate because both versions of the algorithm incur the same expected cost in each round, and $\mathop{\mathbb{E}}\left[\sum_{t=1}^T c_t(\text{\textbf{FPML}})\right]=\sum_{t=1}^T\mathop{\mathbb{E}}\left[c_t(\text{\textbf{FPML}})\right]$. Having new random perturbations in each round is not necessary against oblivious adversaries, but is necessary to achieve the regret bound against adaptive adversaries. 
\end{proof}

\paragraph{Lower bound:}

\begin{proposition}
(Lower bounds) In the full feedback setting, any randomized algorithm has $R^{*}_{T} \geq \Omega\left(\left(\frac{1}{4}\right)^{B}(\log_{2}(N)-\log_{2}(B))\right)$ for $T \geq \Omega(\log_{2}(N)-\log_{2}(B))$. 
\end{proposition}

\begin{proof}
    First suppose $N=2^{k}$ for some $k \in \mathbb{N}$. Let $\mathcal{A}_{0}=\mathcal{A}$. In round $t=1,\dots,T$, we let $\mathcal{A}_{t}$ be a uniformally randomly chosen subset of $\mathcal{A}_{t-1}$ of size $\max(2^{k-t},1)$, and we let $c_{t}(a)=0$ if $a \in \mathcal{A}_{t}$, $1$ otherwise. Suppose an algorithm chooses arms $S_{t} \subset \mathcal{A}$ in round $t$. Then 
    
    \begin{align*}
        P[S_t \cap A_{t} = \emptyset] &\geq \prod_{i=0}^{B-1} \left(1-\frac{|\mathcal{A}_{t}|}{|\mathcal{A}_{t-1}|-i}\right)\\
        &\geq  \left(1-\frac{|\mathcal{A}_{t}|}{|\mathcal{A}_{t-1}|-B}\right)^{B}\\
        &= \left(1-\frac{1}{2-B/|\mathcal{A}_{t}}\right)^{B}\\
        &\geq \left(1-\frac{3}{4}\right)^{B}\\
        &=\left(\frac{1}{4}\right)^{B}\\
    \end{align*}
    
    provided that $|\mathcal{A}_{t}| \geq \frac{3}{2}B$ and $t \leq k$, which holds when $t \leq k-\log_{2}(B)-\log_{2}\left(\frac{3}{2}\right)$. If we set $T=\lfloor k-\log_{2}(B)-\log_{2}\left(\frac{3}{2}\right)\rfloor$, then the expected cost of any fixed algorithm ALG is $\geq \left(\frac{1}{4}\right)^{B}T=\Omega\left(\left(\frac{1}{4}\right)^{B}(\log_{2}(N)-\log_{2}(B))\right)$. By construction, the cost of the best expert in hindsight is $0$. Since the expected regret is $0$, there exists fixed cost functions $c_{1},\dots,c_{T}$ such that the expected regret of ALG on this on this sequence is $\geq \Omega\left(\left(\frac{1}{4}\right)^{B}(\log_{2}(N)-\log_{2}(B))\right)$. If $N$ is not a power of 2, we can just let $\mathcal{A}_0 \subset \mathcal{A}$ be any subset of size $2^{\lfloor \log_{2}(N) \rfloor}$ and the asymptotic bounds remain the same. 
\end{proof}

\paragraph{Proof of Proposition \ref{prop:full-to-partial}}

\begin{proof}
    Fix deterministic cost functions $c_{1},\dots,c_{T}$. We first consider the simpler case where the unbiased cost estimators $(\hat{c}_1,\dots,\hat{c}_{T})$ are jointly independent of any random perturbations used by the algorithm, and the algorithm re-uses random perturbations between rounds, i.e. $p_{t}(a)=p(a)$ for all $a \in \mathcal{A}, t \in [T]$. Afterwards we will show how to reduce the general problem to this special case.
    
    Writing \(\mathcal{F}_t\) for the \(\sigma\)-algebra generated by all actions and observations (as well as any other randomness) up to and including round \(t\), for each \(t \in [T]\) let \(\hat{c}_t\) be a \(\mathcal{F}_t\)-measurable random function \(\mathcal{A} \to [0,K]\) such that $\mathop{\mathbb{E}}[\hat{c}_{t}(a) \mid \mathcal{F}_{t-1}]=c_{t}(a)$ for each \(a\). Assume an oblivious adversary and that w.l.o.g. instead of perturbations there is a `round zero' with costs \((-p(a))_{a \in \mathcal{A}}\) where \(p(a) \sim \mathrm{Exp}(\epsilon)\) independently for each \(a\); define \(\mathcal{F}_0 \coloneqq \sigma((p(a))_{a \in \mathcal{A}}\) to be the \(\sigma\)-algebra generated by these and include it in each \((\mathcal{F}_{t})_{t \ge 1}\).
    
    Writing \(\hat{C}_i(\cdot) \coloneqq \sum_{t=1}^i \hat{c}_t(\cdot)\) for cumulative estimated reward and \(\hat{C}_i^\ast(\cdot) \coloneqq \sum_{t=0}^i \hat{c}_t(\cdot) = \hat{C}_i(\cdot) - p(\cdot)\) for the same but including the `round zero' random initializations,
    define
    \[R'_i \coloneqq \sum_{t=1}^i c_t(S_t) - \min_{a \in \mathcal{A}} \hat{C}_i^\ast(a)\]
    for each \(i \in [T]\).
    Let \(S_i\) be the set of arms chosen by the algorithm at round \(i\) and \(a_i^\ast\) be the best of these by perturbed estimated cost. We follow the argument from \cref{lem:FPML-overtaking-bound}:
    \[
        R'_i - R'_{i-1} = c_i(S_i) - \hat{c}_i(a_i^\ast) + \left(\sum_{t=1}^{i-1} \hat{c}_t(a_{i-1}^\ast) - \sum_{t=1}^{i-1} \hat{c}_t(a_i^\ast)\right) \le c_i(S_i) - \hat{c}_i(a_i^\ast)
    \]
    and so
    \[
        \E[R'_i - R'_{i-1} \mid \mathcal{F}_{i-1}] \le c_i(S_i) - \E[\hat{c}_i(a_i^\ast) \mid \mathcal{F}_{i-1}] = c_i(S_i) - c_i(a_i^\ast) \le \mathds{1}[a_i^\ast \not\in S_i].
    \]
    Hence (using the tower law)
    \begin{align*}
    \E[R'_T \mid \mathcal{F}_0] = \E\left[\sum_{t=1}^T \E[R'_t - R'_{t-1} \mid \mathcal{F}_{t-1}] + R'_0 \right] &\le \E\left[\sum_{t=1}^T \mathds{1}[a_t^\ast \not\in S_t] + R_0 \mid \mathcal{F}_0 \right] \\
    &= \E[|\mathcal{I}| \mid \mathcal{F}_0]+ \max_{a \in \mathcal{A}} p(a),
    \end{align*}
    where \(\mathcal{I} \coloneqq \{t \in [T] : a_t^\ast \not\in S_t\}\). Noting that by Jensen's inequality
    \begin{align*}
        \E[R'_T \mid \mathcal{F}_0] \ge \sum_{t=1}^T c_t(S_t) - \min_{a \in \mathcal{A}}\E[\hat{C}_T(a) \mid \mathcal{F}_0] &= \sum_{t=1}^T c_t(S_t) - \min_{a \in \mathcal{A}} (C_T(a) - p(a)) \\
        &\ge \sum_{t=1}^T c_t(S_t) - C_T(a^\ast) + p(a^\ast)
    \end{align*}
    (where \(a^\ast\) is the best-in-hindsight arm) hence gives that the algorithm regret satisfies
    \[\E[R_T^\ast \mid \mathcal{F}_0] \le \E[|\mathcal{I}| \mid \mathcal{F}_0] + \max_{a \in \mathcal{A}} p(a) - p(a^\ast).\]

    Since \(\E[p(a^\ast)] = 1/\epsilon\) as for any fixed action, and \(\max_{a \in \mathcal{A}} p(a)\) is the maximum of \(|\mathcal{A}|\) i.i.d. \(\mathrm{Exp}(\epsilon)\) random variables, so has expectation at most \((1+\ln|\mathcal{A}|)/\epsilon\) as argued in \citet{DBLP:journals/jcss/KalaiV05}, taking expectations gives
    \[\E[R_T^\ast] \le \frac{\ln|\mathcal{A}|}{\epsilon} + \E[|\mathcal{I}|].\]

    It remains to upper-bound \(\E[|\mathcal{I}|]\). 

    Fix \(t \in [T]\) and let \(V \coloneqq \min_{a \in \mathcal{A} - S_t} \hat{C}_{t-1}^\ast(a)\). So for any \(a\), \(a \in S_t\) iff \(\hat{C}_{t-1}^\ast(a) < V\). Define \(E_a \coloneqq \{\hat{C}_{t-1}^\ast(a) < V - K\}\); if this holds then \(a\) must have been ahead of every action \(a' \not\in S_t\) by at least \(K\) and therefore \textit{cannot} be overtaken by any such action, since the estimated costs are all upper-bounded by \(K\). So
    \[\{a \text{ overtaken by some }a' \not\in S_t\} \subset E_a^c.\]
    Note that
    \begin{align*}
        \{a_t^\ast \not\in S_t\} &= \{\exists a' \in \mathcal{A} - S_t : \forall a \in S_t, \ a' \text{ overtakes } a \text{ at round } t\} \\
            &= \bigcup_{a' \in \mathcal{A} - S_t} \bigcap_{a \in S_t} \{\text{\(a'\) overtakes \(a\) at round \(t\)}\} \\
            &\subset \bigcap_{a \in S_t} \bigcup_{a' \in \mathcal{A} - S_t} \{\text{\(a'\) overtakes \(a\) at round \(t\)}\} = \bigcap_{a \in S_t} \{a \text{ overtaken by some } a' \not\in S_t\}.
    \end{align*}
    
    Let \(\mathcal{G}_t \coloneqq \sigma(S_t,(\hat{C}_{t-1}^\ast(a))_{a \notin S_t})\) be the \(\sigma\)-algebra generated by the random set \(S_t\) and the current perturbed estimated cumulative costs of the actions not in it. So we have
    \begin{align*}
        \P\left(a_t^\ast \not\in S_t \mid \mathcal{G}_t\right) &\le \P\left( \bigcap_{a \in S_t} \{a \text{ overtaken by some } a' \not\in S_t\} \mid \mathcal{G}_t \right) \\
            &\le \P\left(\bigcap_{a \in C} E_a^c \mid \mathcal{G}_t\right) \\
            &= \P\left(\bigcap_{a \in S_t} \{\hat{C}_{t-1}^\ast(a) < V-K\} \mid \mathcal{G}_t\right).
    \end{align*}
    But, since \(V = \min_{a \in \mathcal{A} - S_t} \hat{C}_{t-1}^\ast(a)\), applying \cref{lem:indep} gives us that
    \[\P\left(\bigcap_{a \in S_t} \{\hat{C}_{t-1}^\ast(a) < V-K\} \mid \mathcal{G}_t\right) = \prod_{a \in S_t} \P\left(\hat{C}_{t-1}^\ast(a) < V - K \mid \hat{C}_{t-1}^\ast(a) \le V \right).\]
    By the memoryless property of the exponential distribution, each term here just becomes
    \[
    1 - \P\left(p(a) \ge \hat{C}_{t-1}(a)-V+K \mid p(a) \ge \hat{C}_{t-1}(a) - V\right) \le 1 - \P(p(a) \ge K) = 1 - \e^{-K \epsilon}.
    \]
    Where we have used the assumption that the perturbation $p(a)$ is independent of $\hat{C}_{t-1}$. Thus \(\P(a_t^\ast \not\in S_t \mid \mathcal{G}_t) \le (1 - \e^{-K \epsilon})^{B}\). Since this expression is deterministic and so trivially independent from the \(\sigma\)-algebra \(\mathcal{G}_t\), this immediately implies that \(\P(a_t^\ast \not\in S_t) \le (1 - \e^{-K \epsilon})^{B}\).

    The result then follows, since \(\E[|\mathcal{I}|] = \sum_{t=1}^T \P(a_t^\ast \not\in S_t) \le T(1 - \e^{-K\epsilon})^{B}\).
    
    We now show how to reduce the general problem to a simpler case where the unbiased cost estimates $(\hat{c}_1,\dots,\hat{c}_{T})$ are jointly independent of the perturbations used by the algorithm, and the algorithm re-uses random perturbations between rounds, i.e. $p_{t}(a)=p(a)$ for all $a \in \mathcal{A}, t \in [T]$. Consider the general problem. Let $p_{t}$ be the noise perturbations of the algorithm in round $t$, so $p_{t}(a) \sim \frac{1}{\epsilon}$Exp. Let $\hat{c}_{:t}=(\hat{c}_1,\dots,\hat{c}_{t-1})$ and $S(\hat{c}_{:t},p_{t})\subset \mathcal{A}$ be the $B$ lowest cost-perturbed arms given $\hat{c}_{:t}$, $p_{t}$ (i.e. the arms chosen by the algorithm in round $t$ if cost vectors $\hat{c}_{:t}$ are observed and noise perturbation $p_t$ is chosen). We are guaranteed that $\E[\hat{c}_{t}|p_1,\hat{c}_1,p_2,\hat{c}_2,\dots,p_{t}]=c_t$. The expected regret is
    
    \begin{align*}
    \E_{p_1,\hat{c}_1,p_2,\hat{c}_2,\dots,p_T} \left[\sum_{t=1}^T c_t\left(S(\hat{c}_{:t},p_{t})\right)\right]-\min_{a \in \mathcal{A}} \sum_{t=1}^T c_{t}(a)
    \end{align*}
    
    Focusing on just the first term, and letting $\{p'_t\}_{t=0}^T$ be independent random noise perturbations where $p'_t(a) \sim \frac{1}{\epsilon}$Exp, we have
    
    \begin{align*}
        \E_{p_1,\hat{c}_1,p_2,\hat{c}_2,\dots,p_T} \left[\sum_{t=1}^T c_t\left(S(\hat{c}_{:t},p_{t})\right)\right]\\ 
         & = \sum_{t=1}^T\E_{p_1,\hat{c}_1,p_2,\hat{c}_2,\dots,p_T} \left[ c_t\left(S(\hat{c}_{:t},p_{t})\right)\right]\\
        & = \sum_{t=1}^T\E_{p_1,\hat{c}_1,p_2,\hat{c}_2,\dots,p_t} \left[ c_t\left(S(\hat{c}_{:t},p_{t})\right)\right]\\
        & = \sum_{t=1}^T\E_{p_1,\hat{c}_1,p_2,\hat{c}_2,\dots,p_t} \E_{p'_t} \left[ c_t\left(S(\hat{c}_{:t},p'_{t})\right)\right]\\
        & = \sum_{t=1}^T\E_{p'_t} \left[\E_{p_1,\hat{c}_1,p_2,\hat{c}_2,\dots,p_t}  c_t\left(S(\hat{c}_{:t},p'_{t})\right)\right]\\
        & = \sum_{t=1}^T\E_{p'_t} \left[\E_{p_1,\hat{c}_1,p_2,\hat{c}_2,\dots,p_T}  c_t\left(S(\hat{c}_{:t},p'_{t})\right)\right]\\
        & = \sum_{t=1}^T\E_{p'_0} \left[\E_{p_1,\hat{c}_1,p_2,\hat{c}_2,\dots,p_T}  c_t\left(S(\hat{c}_{:t},p'_{0})\right)\right]\\
        & = \E_{p'_0} \left[\E_{p_1,\hat{c}_1,p_2,\hat{c}_2,\dots,p_T}  \sum_{t=1}^T c_t\left(S(\hat{c}_{:t},p'_{0})\right)\right]
    \end{align*}
    
    Therefore the final expected regret is equal to
    
    \begin{align}
    \E_{p'_0} \left[\E_{p_1,\hat{c}_1,p_2,\hat{c}_2,\dots,p_T}  \sum_{t=1}^T c_t\left(S(\hat{c}_{:t},p'_{0})\right)\right]-\min_{a \in \mathcal{A}} \sum_{t=1}^T c_{t}(a)
    \end{align}
    
    Note the expression $\sum_{t=1}^T c_t\left(S(\hat{c}_{:t},p'_{0})\right)$ is precisely the cost incurred by the algorithm when observing cost estimates $\hat{c}_{:T}$ and using random perturbations $p_0$ in each round, where $\E[\hat{c}_t|p_0,p_1,\hat{c}_1,\dots,p_{t-1}]=c_t$. We therefore conclude that the expected regret is equal to the expected regret of the algorithm in the special case where (a) the algorithm fixes an initial perturbation $p'_{0}$ and uses this randomness for all subsequent rounds and (b) where $p'_{0}$ is jointly independent of $\hat{c}_{:T}$.  
\end{proof}

\section{Section 3: Generalized regret bounds for Online Submodular Function Maximization}

\begin{algorithm}
\caption{OG\textsubscript{hybrid}(\(B\),\(\widetilde{B}\))}\label{alg:og_hybrid}
\begin{algorithmic}

\Require \(B \ge \widetilde{B} \ge 1\). Assume for simplicity that \(\widetilde{B} \mid B\); define \(L \coloneqq B/\widetilde{B}\).
\State Let \(\mathcal{B}_1,\ldots,\mathcal{B}_{L}\) be instances of FPML, each with budget \(\widetilde{B}\).

\For{rounds $t=1,\dots,T$}
    \State Let \(S_{t,0} = \langle \rangle\) be the empty schedule.
    \For{$i=1,\dots,L$}
        \State Use \(\mathcal{B}_{i}\) to choose \(\widetilde{B}\) actions \(a_{(i-1)\widetilde{B}+1}^t,\ldots,a_{i\widetilde{B}}^t\).
        \State Set \(S_{t,i} \coloneqq S_{t,i-1} \oplus \langle a_{(i-1)\widetilde{B}+1}^t, \ldots a_{i\widetilde{B}}^t \rangle\).
    \EndFor
    \State Set \(S_t \coloneqq S_{t,L}\); receive the job \(f_t\).
    \For{$i=1,\dots,L$}
        \State For each action \(a \in \mathcal{A}\) feed back the cost \(c_{t}^{(i)}(a) \coloneqq 1-(f_t(\langle a_{1,t}^\ast, \ldots, a_{i-1,t}^\ast, a \rangle) - f_t(\langle a_{1,t}^\ast,\ldots,a_{i-1,t}^\ast\rangle))\) to FPML instance \(\mathcal{B}_{i}\).
        \State Define \(a_{i,t}^\ast \coloneqq \argmin_{j \in [\widetilde{B}]} c_{t}^{(i)}(a_{(i-1)\widetilde{B}+j})\).
    \EndFor
\EndFor

\end{algorithmic}
\end{algorithm}

\textbf{Proof of \cref{thm:streeter_generalized}} \\

Before proving the theorem, we give a modification to the original result from \cite{DBLP:conf/nips/StreeterG08}. The problem setting they considered was slightly more general:

\begin{dfn}
    Let an action now be an activity-duration pair \(a = (\nu,\tau) \in \mathcal{V} \times (0,\infty) = \mathcal{A}\) for some fixed finited set of \emph{activities} \(\mathcal{V}\).\footnote{We will enforce integer durations so that there are only finitely many possible actions to choose from given a duration constraint.} The \emph{length} \(\ell(S)\) of a schedule \(S \in \mathcal{S}\) is now the sum of the durations of all the actions in \(S\). Write \(S_{\langle i \rangle}\) for the prefix of length \(i\) of schedule \(S\).
\end{dfn}

The algorithm OG they introduced, which takes a budget \(B\) and experts algorithm \(\mathcal{E}\), is given in \cref{alg:og} using our notation for ease of reference.

\begin{algorithm}
\caption{OG(\(B,\mathcal{E}\))}\label{alg:og}
\begin{algorithmic}

\Require \(B \ge 1\).
\State Let \(\mathcal{E}_1,\ldots,\mathcal{E}_{B}\) be instances of experts algorithm \(\mathcal{E}\) (e.g. Hedge).

\For{rounds $t=1,\dots,T$}
    \State Let \(S_{t,0} = \langle \rangle\) be the empty schedule.
    \For{$i=1,\dots,B$}
        \State Use \(\mathcal{E}_{i}\) to choose an action \(a_{i}^t = (\nu,\tau) \in \mathcal{A}\).
        \State With probability \(1/\tau\) set \(S_{t,i} \coloneqq S_{t,i-1} \oplus \langle a_{i}^t \rangle\), otherwise set \(S_{t,i} \coloneqq S_{t,i-1}\).
    \EndFor
    \State Set \(S_t \coloneqq S_{t,B}\); receive the job \(f_t\).
    \State For each \(i \in [B]\) and each action \(a = (\nu,\tau) \in \mathcal{A}\) feed back the cost \(c_{t}^{(i)}(a) \coloneqq (f_t(S_{t,i} \oplus \langle a \rangle) - f_t(S_{t,i}))/\tau\) to experts algorithm \(\mathcal{E}_{i}\).
\EndFor

\end{algorithmic}
\end{algorithm}

We first prove a lemma generalizing Theorem 6 in \citet{DBLP:conf/nips/StreeterG08}:

\begin{lemma} \label{lem:thm6}
    Let \(f\) be any job and let \(\bar{G} = \langle \bar{g}_1,\bar{g}_2,\ldots\rangle\) be an infinite `greedy' schedule satisfying
    \[\frac{f(\bar{G}_j \oplus \bar{g}_j) - f(\bar{G}_j)}{\bar{\tau}_j} \ge \max_{(\nu,\tau) \in \mathcal{V} \times (0,\infty)} \left(\frac{f(\bar{G}_j \oplus \langle (\nu,\tau)\rangle) - f(\bar{G}_j)}{\tau}\right) - \epsilon_j, \qquad j \ge 1\]
    for additive errors \(\epsilon_1,\epsilon_2,\ldots \ge 0\), where \(\bar{g}_j = (\bar{v}_j, \bar{\tau}_j)\) and \(\bar{G}_j = \langle \bar{g}_1,\ldots,\bar{g}_{j-1}\rangle\) for each \(j \ge 1\).
    
    Then for any \(L, B_0 \in \mathbb{N}\) and for \(B_1 \coloneqq \sum_{j=1}^L \bar{\tau}_j\),
    \[f(\bar{G}_{\langle B_1\rangle}) > \left(1-\e^{-B_1/B_0}\right)f(S_{B_0}^\ast) - \sum_{j=1}^L \epsilon_j \bar{\tau}_j\]
    where \(S_{B_0} ^\ast \coloneqq \argmax_{S \in \mathcal{S} : \ell(S) = B_0}f(S)\) is the best schedule of length \(B_0\) for \(f\).
\end{lemma}

\begin{proof}
    For each \(j \in \mathbb{N}\) write \(\Delta_j \coloneqq f(S_{B_0}^\ast) - f(\bar{G}_j)\). By Fact 2 from \citet{DBLP:conf/nips/StreeterG08}, for any \(j \in \mathbb{N}\), \(b > 0\) and \(S \in \mathcal{S}\) with \(\ell(S) \le b\),
    \[f(S) \le f(\bar{G}_j) + b \cdot (s_j + \epsilon_j),\]
    where
    \[s_j \coloneqq \max_{(v, \tau) \in \mathcal{V} \times (0,\infty)} \frac{f(\bar{G}_j \oplus \langle(n,\tau)\rangle) - f(\bar{G}_j)}{\tau} = \frac{f(\bar{G}_j \oplus \bar{g}_j) - f(\bar{G}_j)}{\bar{\tau}_j} = \frac{f(\bar{G}_{j+1})-f(\bar{G}_j)}{\bar{\tau}_j},\]
    so in particular for any \(j\)
    \begin{align}
        f(S_{B_0}^\ast) = \max_{S \in \mathcal{S} : \ell(S) = B_0} f(S) & \le f(\hat{G}_j) + B_0 \cdot (s_j + \epsilon_j) \\
            &= f(\hat{G}_j) + B_0\left(\frac{f(\bar{G}_{j+1})-f(\bar{G}_j)}{\bar{\tau}_j} + \epsilon_j\right) \\
            &= f(\hat{G}_j) + B_0\left(\frac{\Delta_j-\Delta_{j+1}}{\bar{\tau}_j} + \epsilon_j\right),
    \end{align}
    giving \(\Delta_j \le B_0\left(\frac{\Delta_j-\Delta_{j+1}}{\bar{\tau}_j} + \epsilon_j\right)\).

    Rearranging gives \( \Delta_{j+1} \le \Delta_j\left(1 - \frac{\bar{\tau}_j}{B_0}\right) + \bar{\tau}_j \epsilon_j\) for each \(j\), and unrolling this inequality and using that \(1 - \frac{\bar{\tau}_j}{B_0} < 1 \ \forall j\) as in \citet{DBLP:conf/nips/StreeterG08} gives us
    \[\Delta_{L+1} \le \Delta_1\left(\prod_{j=1}^L 1 - \frac{\bar{\tau}_j}{B_0}\right) + \sum_{j=1}^L \bar{\tau}_j\epsilon_j.\]
    By definition \(B_1 = \sum_{j=1}^L \bar{\tau}_j\epsilon_j\), and maximizing the product above subject to this constraint results in \(\bar{\tau}_j = \frac{B_1}{L}\) for all \(j\). Thus
    \[\prod_{j=1}^L 1 - \frac{\bar{\tau}_j}{B_0} \le \prod_{j=1}^L 1 - \frac{B_1/L}{B_0} = \left(1 + \frac{(-B_1/B_0)}{L}\right)^L < \e^{-B_1/B_0}\]
    and so
    \[f(S_{B_0}^\ast) - f(\bar{G}_{L+1}) = \Delta_{L+1} < \Delta_1 \e^{-T_1/T_0} + \sum_{j=1}^L \bar{\tau}_j\epsilon_j \le f(S_{B_0}^\ast) \e^{-B_1/B_0} + \sum_{j=1}^L \bar{\tau}_j\epsilon_j,\]
    giving \(f(\bar{G}_{\langle T_1\rangle}) = f(\bar{G}_{L+1}) > (1-\e^{-B_1/B_0})f(S_{B_0}^\ast) - \sum_{j=1}^L \bar{\tau}_j\epsilon_j\) as required.
\end{proof}

Next we prove a generalized regret bound for the original OG algorithm:

\begin{lemma} \label{thm:streeter}
    For \(B \ge B' \log T\) the algorithm OG, run using Hedge as the subroutine experts algorithm, produces a sequence of schedules \(S_1,\ldots,S_B\) with regret
    \[\E\left[\sum_{t=1}^T f_t(S_{B'}^\ast) - \sum_{t=1}^T f_t(S_t)\right] = \mathcal{O}\left(\E\left[\sum_{j=1}^B R_{T,1}(\mathcal{E}_j)\right]\right)\] 
    relative to \(S_{B'}^\ast \coloneqq \argmax_{S \in \mathcal{S} : \ell(S) = B'} \sum_{t=1}^T f_t(S)\), the best-in-hindsight fixed schedule of length \(B'\), where \(R_{T,1}(\mathcal{E}_j)\) is the 1-regret incurred by the \(j\)\textsuperscript{th} experts algorithm.
    
    In particular, when run with Hedge as the subroutine experts algorithm, this is \(\mathcal{O}\left(\sqrt{BT\log N}\right)\).
\end{lemma}

\begin{proof}
    Consider the quantity \(\rho_{B,B'} \coloneqq \left(1 - \e^{-B/B'}\right) \sum_{t=1}^T f_t(S_{B'}^\ast) - \sum_{t=1}^T f_t(S_t)\). As argued in \citet{DBLP:conf/nips/StreeterG08}, we may view the sequence of actions \(a_i^1,\ldots,a_i^T\) selected by each experts algorithm \(\mathcal{E}_i\) as a single `meta-action' \(\tilde{a}_i \in \mathcal{A}^T\); so the schedules \(S_1,\ldots,S_T\) output by \textbf{OG} can be viewed as a single `meta-schedule' \(\tilde{S} = \langle \tilde{a}_1,\ldots,\tilde{a}_B\rangle\) over \(\mathcal{A}^T\) which is a version of the greedy schedule \(\bar{G}_{B+1}\) for the job \(f = \frac{1}{T}\sum_{t=1}^T f_t\), and it may be assumed that each meta-action \(\tilde{a}_t\) takes unit time per job. Thus we may write
    \[\rho_{B,B'} = T\left[\left(1 - \e^{-B/B'}\right)f(S_{B'}^\ast) - f(\tilde{S})\right]\]
    (after extending the domain of \(f\) appropriately). Applying \cref{lem:thm6} with \(L=B\), \(B_0 = B'\), \(B_1 = \sum_{j=1}^{B} \bar{\tau}_j = B\) (by the unit-time assumption) then immediately gives
    \[\rho_{B,B'} < T\sum_{j=1}^{B} \bar{\tau_j}\epsilon_j = T\sum_{j=1}^{B} \epsilon_j.\]
    Taking expectations,
    \[\E[\rho_{B,B'}] \le T\sum_{j=1}^{B} \E[\epsilon_j] = T\sum_{j=1}^{B} \E\left[\frac{R_{T,1}(\mathcal{E}_j)}{T}\right]\]
    where \(R_{T,1}(\mathcal{E}_j)\) is the 1-regret incurred by the \(j\)\textsuperscript{th} experts algorithm; here we used that \(\E[\epsilon_j] = \E[R_{T,1}(\mathcal{E}_j)/n]\) as argued in \citet{DBLP:conf/nips/StreeterG08}. So \(\E[\rho_{B,B'}] \le \E\left[\sum_{j=1}^{B} R_{T,1}(\mathcal{E}_j)\right]\).
    
    The result then follows quickly: since \(B \ge B' \log T\), so \(\e^{-B/B'} \le \e^{-\ln T} = T^{-1}\). Thus
    \[\rho_{B,B'} \ge \left(1 - T^{-1}\right) \sum_{t=1}^T f_t(S_{B'}^\ast) - \sum_{t=1}^T f_t(S_t) = \mathcal{R}_{B'} - T^{-1} \sum_{t=1}^T f_T(S_{B'}^\ast)\]
    where \(\mathcal{R}_{B'} \coloneqq \sum_{t=1}^T f_t(S_{B'}^\ast) - \sum_{t=1}^T f_t(S_t)\) is the regret of interest. Consequently,
    \[\mathcal{R}_{B'} \le \rho_{B,B'} + T^{-1} \sum_{t=1}^T f_t(S_{B'}^\ast) \le \rho_{B,B'} + T^{-1} \cdot T = \rho_{B,B'} + 1.\]
    and the result follows.
    
    The bound \(\E\left[\sum_{j=1}^{B} R_{T,1}(\mathcal{E}_j)\right] = \mathcal{O}(\sqrt{BT\log N})\) when using Hedge was shown in \citet{DBLP:conf/nips/StreeterG08}.
\end{proof}

Finally we prove the theorem on OG\textsubscript{hybrid}:

\begin{proof}[Proof of \cref{thm:streeter_generalized}]
    Note first that under \cref{additional_assumption}, any job \(f\), any schedule \(S \in \mathcal{S}\) and any sub-schedule \(S'\) of \(S\) (i.e. the actions of \(S'\) appear in order in \(S\)) satisfy
    \begin{equation*}
    f(S) \ge f(S'); \label{eq:monotonicity_general}
    \end{equation*}
    this is immediate using monotonicity and induction.

    Suppose for each \(i\in[L]\) there is a fictional experts algorithm (classical full feedback multi-armed bandit algorithm) \(\mathcal{E}_{i}\) which picks \(a^\ast_{i,t}\) at each round \(t\), and consider a hypothetical instance of the standard algorithm \textbf{OG} run with time allowance \(L\) and these fictional experts algorithms \(\mathcal{E}_1,\ldots,\mathcal{E}_{L}\) as subroutines.

    Since \(L \ge B'\log T\) (by our assumption that \(B \ge B' \widetilde{B} \log T\)), by \cref{thm:streeter} the \(B'\)-regret of our \(\mathbf{OG}\) instance is upper-bounded in expectation by \(\sum_{i=1}^{L}\mathcal{R}_1(\mathcal{E}_{i})\), where \(\mathcal{R}_1(\mathcal{E}_i)\) is the total 1-regret experienced by \(\mathcal{E}_i\).
    
    But the payoff received by this \textbf{OG} instance at each round \(t\) is \(f(\langle a_{1,t}^\ast, \ldots, a_{L,t}^\ast\rangle)\), which by \cref{eq:monotonicity_general} is upper-bounded by \(f(S_t)\), the payoff of \textbf{OG\textsubscript{hybrid}}, since the actions \(a_{1,t}^\ast,\ldots,a_{L,t}^\ast\) appear in order in \(S_t\). So the \(B'\)-regret \(\mathcal{R}_{B'}\) of \textbf{OG\textsubscript{hybrid}} must be at most that of our fictional \textbf{OG} instance, giving the upper bound
    \[\E[\mathcal{R}_{B'}] \le \sum_{i=1}^{L}\E[\mathcal{R}_1(\mathcal{E}_{i})]. \label{eq:streeter_general_sumbound}\]
    
    It remains to argue how large each of the regret of each of these `fictional' experts algorithms \(\mathcal{E}_{i}\) is. Writing \(a^{\ast\ast}_{i} = \argmin_{a \in \mathcal{A}} \sum_{t=1}^T c_{t}^{(i)}(a)\) for the best-in-hindsight fixed action under the costs passed to these subroutines, the regret incurred by \(\mathcal{E}_{i}\) is therefore
    \begin{align}
        \mathcal{R}_1(\mathcal{E}_{i}) &= \sum_{t=1}^T c_{t}^{(i)}(a^\ast_{i,t}) - \sum_{t=1}^T c_{t}^{(i)}(a^{\ast\ast}_i) \\
            &= \sum_{t=1}^T \max_{j\in[\widetilde{B}]} c_{t}^{(i)}(a_{(i-1)\widetilde{B}+j}^t) - \sum_{t=1}^T c_{t}^{(i)}(a^{\ast\ast}_i) = \mathcal{R}_1(\mathcal{B}_i).
    \end{align}
    where \(\mathcal{R}_1(\mathcal{B}_i)\) is the 1-regret incurred by multitasking bandit algorithm \(\mathcal{B}_i\). So by \cref{eq:streeter_general_sumbound}
    \[\E[\mathcal{R}_{B'}] \le \sum_{i=1}^{L} \E[\mathcal{R}_1(\mathcal{B}_i)] = L \E[\mathcal{R}_1(\mathcal{B})] = \frac{B \E[\mathcal{R}_1(\mathcal{B})]}{L},\]
    where \(\E[\mathcal{R}_1(\mathcal{B})]\) is the expected 1-regret of any of the instances \(\mathcal{B}_1,\ldots,\mathcal{B}_{L}\) of \(\mathcal{B}\).
\end{proof}

\paragraph{Proof of \cref{prop:top_B_regret}}

\begin{proof}[Sketch proof]
    This is a special case of the more general result that the expected regret relative to the best-in-hindsight fixed set of size \(B'\) is at most
    \[\frac{1-B'^{-1} + \ln\left(N/B'\right)}{\epsilon} + T\sum_{j=0}^{B'-1} \binom{B}{j} \e^{-j\epsilon}(1-\e^{-\epsilon})^{B-j} + \mathtt{err}_{B'}\]
    where \(\mathtt{err}_{B'}\) is the difference in cumulative cost between the best-in-hindsight set of \(B'\) actions and the set of the top \(B'\) actions in hindsight on the given problem instance.
    
    The proof of this is a simple adaptation of the 1-regret argument, using ``an action not in the top \(B\) enters the best \(B'\)-set" as the event of interest; use the harmonic series form of the expectation of the max of exponential random variables to get a lower bound, and use a binomial counting argument to bound the probability of the event.
\end{proof}

\section{Experiments}

\subsection{Full comparison of \texorpdfstring{OG\textsubscript{hybrid}}{OG\_hybrid}}

In \cref{tab:streeter_bbo} we give a more detailed comparison of FPML and OG with various instantiations of OG\textsubscript{hybrid} on the hyperparameter-selection task from \cref{sec:experiments}. Specifically, we include for each \(B\) and each possible pair \((B_1,B_2)\) s.t. \(B_1 B_2 = B\) a version of OG\textsubscript{hybrid} with \(B_2\) internal boxes and arm budget \(B_1\) per box. As can be seen, in all cases decreasing the greediness and adding more arms per box is beneficial in this application.

\renewcommand{\c}[1]{#1}
\renewcommand{\d}[1]{#1}
\begin{table}
    \centering
    \scriptsize
    \caption{Sample means and standard deviations of normalized validation scores of FPML, OG\textsubscript{hybrid} and OG over black-box optimizers.}
    \label{tab:streeter_bbo}
    \begin{subtable}{0.49\textwidth}
    \centering
    \caption{\(B=1\)}
    \begin{tabularx}{\textwidth}{r *2{|Y}}
        \multicolumn{1}{c}{} & \multicolumn{1}{c}{\textbf{Mean}} & \multicolumn{1}{c}{\textbf{StD}} \\
        \toprule
        \textbf{Best in hindsight} & \c{0.574} & \d{0} \\
        \textbf{FPML} & \c{0.426} & \d{0.0202} \\
        \textbf{Exp3} & \c{0.351} & \d{0.0194} \\
        \bottomrule
    \end{tabularx}
    \end{subtable}
    \hfill
    \begin{subtable}{0.49\textwidth}
    \centering
    \caption{\(B=2\)}
    \begin{tabularx}{\textwidth}{r *2{|Y}}
        \multicolumn{1}{c}{} & \multicolumn{1}{c}{\textbf{Mean}} & \multicolumn{1}{c}{\textbf{StD}} \\
        \toprule
        \textbf{Best in hindsight} & \c{0.710} & \d{0} \\
        \textbf{FPML} & \c{0.652} & \d{0.0194} \\
        \textbf{OG\textsubscript{hybrid}} (\((B_1,B_2)=(1,2)\)) & \c{0.577} & \d{0.0187} \\
        \textbf{OG} & \c{0.519} & \d{0.0179} \\
        \bottomrule
    \end{tabularx}
    \end{subtable}
    \\
    \vspace{2em}
    \begin{subtable}{0.49\textwidth}
    \centering
    \caption{\(B=3\)}
    \begin{tabularx}{\textwidth}{r *2{|Y}}
        \multicolumn{1}{c}{} & \multicolumn{1}{c}{\textbf{Mean}} & \multicolumn{1}{c}{\textbf{StD}} \\
        \toprule
        \textbf{Best in hindsight} & \c{0.779} & \d{0} \\
        \textbf{FPML} & \c{0.751} & \d{0.0151} \\
        \textbf{OG\textsubscript{hybrid} (\((B_1,B_2)=(1,3)\))} & \c{0.657} & \d{0.0191} \\
        \textbf{OG} & \c{0.617} & \d{0.0166} \\
        \bottomrule
    \end{tabularx}
    \end{subtable}
    \hfill
    \begin{subtable}{0.49\textwidth}
    \centering
    \caption{\(B=4\)}
    \begin{tabularx}{\textwidth}{r *2{|Y}}
        \multicolumn{1}{c}{} & \multicolumn{1}{c}{\textbf{Mean}} & \multicolumn{1}{c}{\textbf{StD}} \\
        \toprule
        \textbf{Best in hindsight} & \c{0.836} & \d{0} \\
        \textbf{FPML} & \c{0.813} & \d{0.0108} \\
        \textbf{OG\textsubscript{hybrid} (\((B_1,B_2)=(2,2)\))} & \c{0.756} & \d{0.0149} \\
        \textbf{OG\textsubscript{hybrid} (\((B_1,B_2)=(1,4)\))} & \c{0.716} & \d{0.0178} \\
        \textbf{OG} & \c{0.689} & \d{0.0151} \\
        \bottomrule
    \end{tabularx}
    \end{subtable}
    \\
    \vspace{2em}
    \begin{subtable}{0.49\textwidth}
    \centering
    \caption{\(B=5\)}
    \begin{tabularx}{\textwidth}{r *2{|Y}}
        \multicolumn{1}{c}{} & \multicolumn{1}{c}{\textbf{Mean}} & \multicolumn{1}{c}{\textbf{StD}} \\
        \toprule
        \textbf{Best in hindsight} & \c{0.874} & \d{0} \\
        \textbf{FPML} & \c{0.855} & \d{0.0094} \\
        \textbf{OG\textsubscript{hybrid} (\((B_1,B_2)=(1,5)\))} & \c{0.756} & \d{0.0150} \\
        \textbf{OG} & \c{0.734} & \d{0.0140} \\
        \bottomrule
    \end{tabularx}
    \end{subtable}
    \hfill
    \begin{subtable}{0.49\textwidth}
    \centering
    \caption{\(B=6\)}
    \begin{tabularx}{\textwidth}{r *2{|Y}}
        \multicolumn{1}{c}{} & \multicolumn{1}{c}{\textbf{Mean}} & \multicolumn{1}{c}{\textbf{StD}} \\
        \toprule
        \textbf{Best in hindsight} & \c{0.901} & \d{0} \\
        \textbf{FPML} & \c{0.888} & \d{0.0072} \\
        \textbf{OG\textsubscript{hybrid} (\((B_1,B_2)=(3,2)\))} & \c{0.836} & \d{0.0111} \\
        \textbf{OG\textsubscript{hybrid} (\((B_1,B_2)=(2,3)\))} & \c{0.814} & \d{0.0143} \\
        \textbf{OG\textsubscript{hybrid} (\((B_1,B_2)=(1,6)\))} & \c{0.785} & \d{0.0137} \\
        \textbf{OG} & \c{0.767} & \d{0.0157} \\
        \bottomrule
    \end{tabularx}
    \end{subtable}
\end{table}

\subsection{Synthetic tasks}

In this section we evaluate our algorithms on three synthetic tasks. In all cases,
\begin{itemize}
    \item let \(S^\ast\) be the best-in-hindsight set of \(B\) arms;
    \item let \(S_{\mathrm{greedy}}\) be the greedy choice of \(B\) arms in hindsight;
    \item let \(S_{\mathrm{top}}\) be the top \(B\) arms in hindsight.
\end{itemize}

\paragraph{Task 1:} The first environment is one where \(S^\ast = S_{\mathrm{greedy}}\) and this set does better than \(S_{\mathrm{top}}\); greediness is better than picking the top \(B\) arms. There are \(|\mathcal{A}|=15\) available arms and two types of round, \(A\) and \(B\), which occur with equal probability; costs are distributed within each round according to \cref{tab:syntha_costs}. So the best fixed arm set of any size up to 10 will be split evenly across arms \(\{1,2,3,4,5\}\) and arms \(\{11,12,13,14,15\}\)---and will be the greedy choice---but for \(B \le 5\) the top \(B\) arms will always be in \(\{1,2,3,4,5\}\). We see in \cref{fig:syntha} that FPML does not outperform the greedy algorithms on this task.

\paragraph{Task 2:} The second environment is one where (approximately) \(S^\ast = S_{\mathrm{greedy}} = S_{\mathrm{top}}\); greediness is good but no better than picking the top \(B\) arms. There are \(|\mathcal{A}|=10\) available arms and costs are distributed according to \cref{tab:synthb_costs}; because there are no groups of anticorrelated actions, the performance gap between the best set and the top \(B\) arms is trivially small. The results in \cref{fig:synthb} show that FPML outperforms the greedy algorithms on this task.

\paragraph{Task 3:} The third environment is one where \(S^\ast = S_{\mathrm{top}}\) and this set does better better than \(S_{\mathrm{greedy}}\); greediness is worse than just picking the top \(B\) arms. Suppose there are \(|\mathcal{A}|=4\) available arms and a budget of \(B=3\). Costs are deterministic and listed in \cref{tab:synthc_costs} for some parameter \(\delta\) which we set to 0.01. The top 3 arms are \(S_{\mathrm{top}} = \{1,3,4\}\) and this is also the best-in-hindsight set \(S^\ast\), incurring minimum cost 0 at each round. A quick calculation shows that the greedy choice \(S_{\mathrm{greedy}}\) is either \(\{1,2,3\}\) or \(\{1,2,4\}\), though, and either of these sets incur an average minimum cost of \(1/8-\delta/4\), substantially higher. Our empirical results in \cref{tab:synthc} show this gap in practice.

\begin{table}
    \begin{center}
    \caption{Cost distributions for round types \(A\) and \(B\) in the first synthetic environment; Beta distributions are parameterized by mean and variance, not shape.}
    \label{tab:syntha_costs}
    \begin{tabular}{r|c|c|c}
        \toprule
        \textbf{Arm} & \textbf{\(A\)-rounds} & \textbf{\(B\)-rounds} & \textbf{Resulting mean} \\
        \midrule
        \textbf{Actions 1 to 5} & \(\mathrm{Beta}(0.4,0.01)\) & Always 1 & \c{0.7} \\
        \textbf{Actions 6 to 10} & \(\mathrm{Beta}(0.6,0.01)\) & Always 1 & \c{0.8} \\
        \textbf{Actions 11 to 15} & Always 1 & \(\mathrm{Beta}(0.8,0.01)\) & \c{0.9} \\
        \bottomrule
    \end{tabular}
    \end{center}
\end{table}

\begin{figure}
    \centering
    \includegraphics[width=0.7\textwidth]{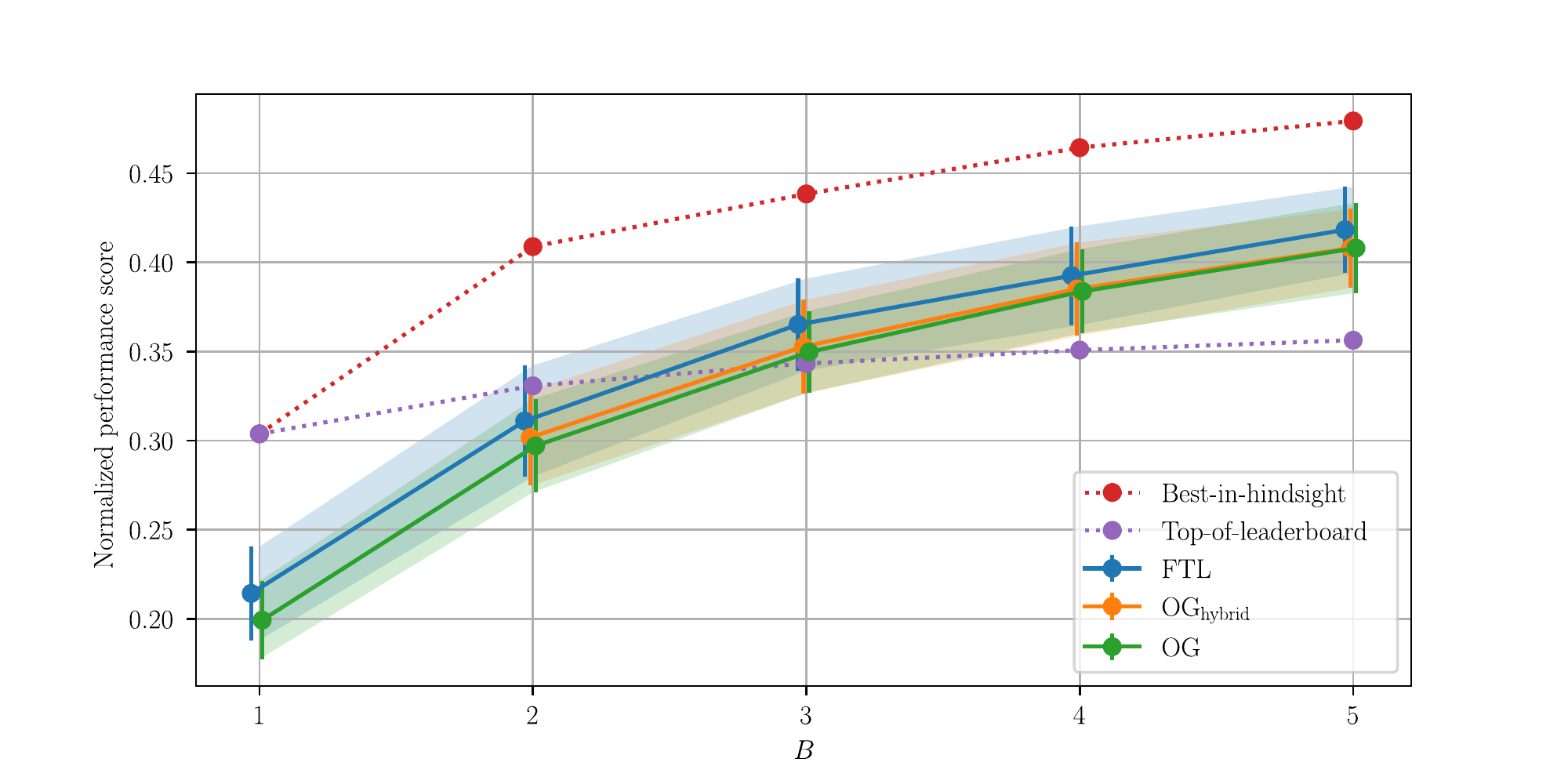}
    \caption{Performances (1-cost) on the first synthetic task. OG\textsubscript{hybrid} is run with FPML boxes each with arm budget 1.}
    \label{fig:syntha}
\end{figure}

\begin{table}
    \begin{center}
    \caption{Cost distributions in the second synthetic environment.}
    \label{tab:synthb_costs}
    \begin{tabular}{r|c}
        \toprule
        \textbf{Arm} & \textbf{Distribution} \\
        \midrule
        \textbf{1} & \(\mathrm{Beta}(0.4,0.01)\) \\
        \textbf{2} & \(\mathrm{Beta}(0.45,0.01)\) \\
        \textbf{3} & \(\mathrm{Beta}(0.5,0.01)\) \\
        \textbf{4} & \(\mathrm{Beta}(0.55,0.01)\) \\
        \textbf{5} & \(\mathrm{Beta}(0.6,0.01)\) \\
        \textbf{6} & \(\mathrm{Beta}(0.65,0.01)\) \\
        \textbf{7} & \(\mathrm{Beta}(0.7,0.01)\) \\
        \textbf{8} & \(\mathrm{Beta}(0.75,0.01)\) \\
        \textbf{9} & \(\mathrm{Beta}(0.8,0.01)\) \\
        \textbf{10} & \(\mathrm{Beta}(0.85,0.01)\) \\
        \bottomrule
    \end{tabular}
    \end{center}
\end{table}

\begin{figure}
    \centering
    \includegraphics[width=0.7\textwidth]{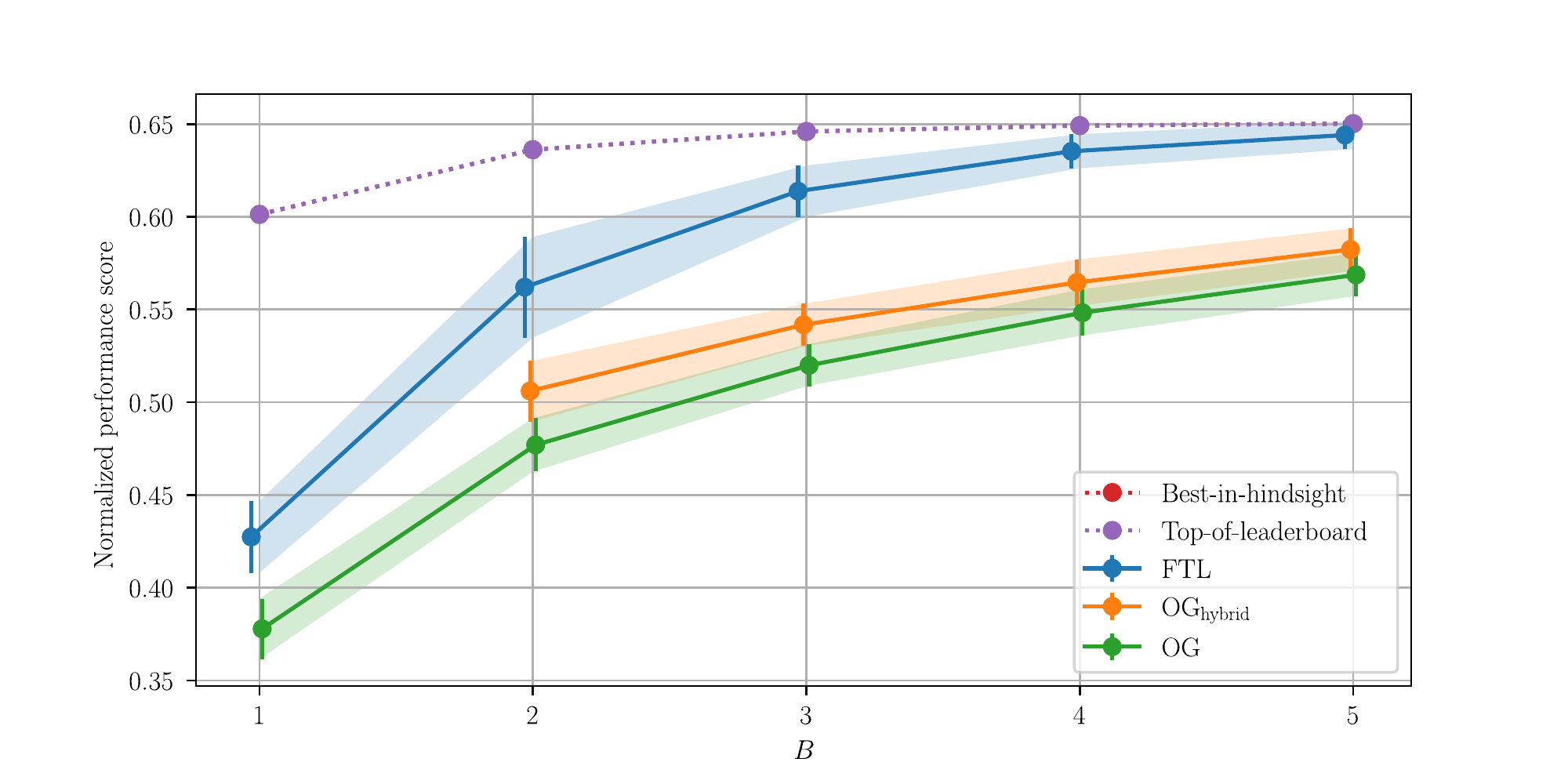}
    \caption{Performances (1-cost) on the second synthetic task. OG\textsubscript{hybrid} is run with FPML boxes each with arm budget 1.}
    \label{fig:synthb}
\end{figure}

\begin{table}
    \begin{center}
    \caption{Costs in the third synthetic environment, for some parameter \(\delta \in (0,1/2)\).}
    \label{tab:synthc_costs}
    \begin{tabularx}{\textwidth}{r *5{|Y}}
        \toprule
        \textbf{Arm} & \multicolumn{4}{c|}{\textbf{Reward at rounds \(i \equiv k \mod 4\) for...}} & \textbf{Average} \\
         & \(k = 1\) & \(k=2\) & \(k=3\) & \(k=4\) & \textbf{cost}\\
        \midrule
        \textbf{1} & \(1-\delta\) & \(1-\delta\) & 0 & 0 & \(1/2-\delta/2\) \\
        \textbf{2} & \(1/2-\delta\) & \(1/2-\delta\) & 1 & 1 & \(3/4-\delta/2\) \\
        \textbf{3} & 0 & 1 & 0 & 1 & \(1/2\) \\
        \textbf{4} & 1 & 0 & 1 & 0 & \(1/2\) \\
        \bottomrule
    \end{tabularx}
    \end{center}
\end{table}

\begin{table}
    \begin{center}
    \caption{Means and standard deviations over 50 trials of performances (1-cost) for various combinations of FPML and online greedy algorithms in the third synthetic environment, with \(\delta=0.01\).}
    \label{tab:synthc}
    \begin{tabular}{r|c|c}
        \toprule
        \textbf{Algorithm} & \textbf{Mean} & \textbf{StD} \\
        \midrule
        \textbf{Best-in-hindsight} & \c{1.000} & \d{0} \\
        \textbf{Top-of-leaderboard} & \c{1.000} & \d{0} \\
        \textbf{FPML} & \c{0.964} & \d{0.0145} \\
        \textbf{OG\textsubscript{hybrid}} (\((B_1,B_2) = 1,3)\)) & \c{0.823} & \d{0.0200} \\
        \textbf{OG} & \c{0.799} & \d{0.0202} \\
        \bottomrule
    \end{tabular}
    \end{center}
\end{table}

\subsection{Geometric resampling}

The \textit{geometric resampling} technique used in the second and third partial feedback versions of FPML in the experiments is adapted from \citet{neu2013efficient}. At each round cost estimates
\[\hat{c}_t(a) \coloneqq \begin{cases} \frac{c_t(a)}{ \hat{q}_{t,a}} & \text{if \(a\) was pulled}, \\ 0 & \text{otherwise} \end{cases} \]
are made, where \(\hat{q}_{t,a}\) is an estimate of the probability \(q_{t,a} \coloneqq \P(\text{arm \(a\) pulled at round \(t\)})\). These estimates are made by sampling \(\frac{1}{\hat{q}_{t,a}} \sim \mathrm{Geom}(q_{t,a})\), which is done by repeating the algorithm's execution at this round and counting how many trials are needed until \(a\) is pulled again. In practice, the number of repetitions must be capped and this introduces some bias to the estimates, but this is not problematic in practice. In fact, there is a bias variance trade-off, because $K=\max_{a \in \mathcal{A}} |\hat{c}_{t}(a)|$ is bounded by the number of samples we take. Therefore more samples lead to lower bias but higher variance. Using bounds similar to those of Proposition \ref{prop:full-to-partial} as a guide (the bounds of Proposition 5 were subsequently refined after the experiments were concluded), we picked the number of samples to be $\left(N\left(\frac{TN}{\ln(N)}\right)^{\tilde{B}}\right)^{1/(2\tilde{B}+1)}$, so  $K=\left(N\left(\frac{TN}{\ln(N)}\right)^{\tilde{B}}\right)^{1/(2\tilde{B}+1)}$  and $\epsilon =  \left(\frac{\ln(N)}{T}\left(\frac{\ln(N)}{TN}\right)^{\tilde{B}}\right)^{1/(2\tilde{B}+1)}$, where $\tilde{B}$ is the budget of each \textbf{FPML-partial} box.

These estimators make complete use of the information received at each round, unlike the simple one-arm uniform sampling, $B$ arms exploiting version of FPML with partial feedback mentioned in \cref{sec:FPML}. Moreover, the construction of cost estimates means no explicit exploration is necessary; an arm that hasn't been pulled for several rounds will be overtaken in estimated cumulative cost by ones that have, and so will eventually be pulled again, thus inducing a self-stabilizing property that would not occur if we used the same technique to estimate rewards \(r_t(a) \coloneqq 1 - c_t(a)\) instead.

\subsection{Methods}

\paragraph{Reward definitions:} For the black-box optimization experiments in \cref{sec:experiments}, the \textit{reward} ($1-$cost) for each black-box optimizer on each machine learning task (i.e. round) was defined as follows. This approach was inspired heavily by the Bayesmark package used in the 2020 NeurIPS BBO Challenge and which we based our implementation on \citep{bayesmark}.

Fix a round \(t\) and an optimizer \(a\). Let \(\mathtt{opt}_t\) be an estimate of the global minimum classification/regression loss achievable (at validation, not test) on the task corresponding to round \(t\). Define \(\overline{\mathtt{rand}}_t\) to be the mean performance of a random hyperparameter search on this task (i.e. the smallest loss achieved using any hyperparameter in the random search, averaged over trials).\footnote{In reality this is estimated using a more statistically efficient technique than actually performing the random search, as in the Bayesmark package.} Finally, let \(\overline{\mathtt{loss}}_t(a)\) be the actual averaged minimum loss of the optimizer \(a\) on this problem.

The reward is then defined as
\[r_i(a) \coloneqq \frac{\overline{\mathtt{loss}}_t(a) - \mathtt{opt}_t}{\overline{\mathtt{rand}}_t(a) - \mathtt{opt}_t}.\]
Conceptually, the reward is 0 when optimizer \(a\) performs as badly as a random search, and 1 when it performs as well as is possible on this task.

As per usual, the reward for a bandit algorithm selecting multiple optimizers at each round is then calculated as the maximum of the rewards of each optimizer (equivalent to the minimum of costs).

\paragraph{Bayesian optimizers used:} The nine black-box optimization algorithms we ran the experiments in \cref{sec:experiments} over were as follows:

\begin{enumerate}
    \item Hyperopt \citep{bergstra2015hyperopt}
    \item The \texttt{AUCBanditMetaTechniqueA} technique from OpenTuner \citep{ansel2014opentuner}
    \item The \texttt{PSO\_GA\_Bandit} technique from OpenTuner \citep{ansel2014opentuner}
    \item The \texttt{PSO\_GA\_DE} technique from OpenTuner \citep{ansel2014opentuner}
    \item PySOT \citep{eriksson2019pysot}
    \item Scikit-Optimize \citep{tim_head_2018_1207017} using base estimator \texttt{GBRT} and acquisition objective \texttt{gp\_hedge}
    \item Scikit-Optimize \citep{tim_head_2018_1207017} using base estimator \texttt{GP} and acquisition objective \texttt{gp\_hedge}
    \item Scikit-Optimize \citep{tim_head_2018_1207017} using base estimator \texttt{GP} and acquisition objective \texttt{LCB}
    \item Random search
\end{enumerate}

The default settings of each package were used.

\section{Section \ref{sec:lp}: Linear Programming}

\paragraph{Proof of Proposition \ref{prop:lp}}

(Note: this proof was given for $B=1$ in \cite{DBLP:journals/toc/AroraHK12} with slightly tighter bounds, and essentially remains unchanged for $B \geq 1$).

\begin{proof}
    We run the FPML oracle with budget $B$, $N=n$ arms, and $\epsilon=((\ln(N)+1)/T)^{1/(B+1)}$. In round $t \in [T]$ we do the following: Let $d_{t}$ be the joint distribution over $N$ arms returned the FPML oracle in this round. We pass $d_{t}$ to the $(\rho,B)$-bounded oracle, and receive either a vector $x_t \in P$ or that no $x_t$ exists which satisfies the oracle problem. Let us first suppose that we always receive an $x_t$ for each round. Then define the cost function $c_t(i):=A_i x_t - b_i \in [-\rho,\rho]$ and pass this to FPML. After $T$ rounds, and by scaling and translating the cost functions to lie in $[0,1]$, Theorem \ref{thm:FPML-upper-bound} implies that $\forall j \in [N]$ 
    
    \begin{align*}
        \frac{\sum_{t=1}^{T} \mathop{\mathbb{E}}_{(i_1,\dots,i_B) \sim d_t}\left[\min_{i \in \{i_1,\dots,i_B\}} A_{i}x-b_{i}\right]}{T} \leq \frac{4\rho T^{\frac{1}{B+1}}(1+\ln(N))^{\frac{B}{B+1}}}{T} + \frac{\sum_{t=1}^T A_j x_t -b_j}{T}
    \end{align*}
    
    By assumption of the $(\rho,B)$-bounded oracle, the left hand side is $\geq 0$. When $T \geq \left(\frac{1}{\epsilon}\right)^{\frac{B+1}{B}} (4\rho)^{\frac{B+1}{B}} (1+\ln(N))$, it follows that $x:=\frac{\sum_{t=1}^T x_t}{T}$ satisfies $\forall j \in [N], A_j x \geq b_j-\epsilon$. Since $P$ is convex, $x \in P$ and we are done. Now suppose that in some round $t$ we were told the oracle problem was not solvable. We claim that we can conclude that the problem is not feasible and we are done. This is because if $\exists x \in P$ s.t. $Ax \geq b$, then $\mathop{\mathbb{E}}_{(i_1,\dots,i_B) \sim d}\left[\min_{i \in \{i_1,\dots,i_B\}} A_{i}x-b_{i}\right] \geq \mathop{\mathbb{E}}_{(i_1,\dots,i_B) \sim d}\left[\min_{i \in \{i_1,\dots,i_B\}} 0\right]=0$ and so the oracle problem would be solvable. 
\end{proof}

\end{document}